\documentclass{article}

\usepackage[main, final]{neurips_2026}

\usepackage[utf8]{inputenc} 
\usepackage[T1]{fontenc}    
\usepackage{hyperref}       
\usepackage{url}            
\usepackage{booktabs}       
\usepackage{amsfonts}       
\usepackage{nicefrac}       
\usepackage{microtype}      
\usepackage{xcolor}         
\usepackage{hyperref}
\usepackage{url}
\usepackage{amsmath}
\usepackage{tikz}
\usetikzlibrary{arrows.meta, positioning, shapes.geometric}
\usepackage{graphicx}
\usepackage{subcaption}
\usepackage{caption}
\usepackage{tikz}
\usepackage{subcaption}
\usepackage{xcolor}
\usepackage{amsthm}

\newtheorem{lemma}{Lemma}
\newtheorem{Corollary}{Corollary}

\definecolor{KEcolor}{HTML}{76B7B2}
\definecolor{OKEcolor}{HTML}{F28E2B}
\definecolor{OKFcolor}{HTML}{59A14F }
\definecolor{KFcolor}{HTML}{E15759}
\definecolor{KNcolor}{HTML}{4E79A7}
\definecolor{LSTMcolor}{HTML}{EDC948}
\definecolor{OBScolor}{HTML}{B07AA1}
\usepackage{listings}
\usepackage{tcolorbox}
\tcbuselibrary{listings, skins}
\definecolor{pybggray}{RGB}{245,245,245}     
\definecolor{pykeyword}{RGB}{0,128,0}        
\definecolor{pystring}{RGB}{34,139,34}       
\definecolor{pycomment}{RGB}{41,171,135}     
\definecolor{pynumber}{RGB}{0,128,0}      
\definecolor{pyoperator}{RGB}{138,43,226}    
\definecolor{pyfunction}{RGB}{0,0,200}       
\definecolor{MidBlue}{RGB}{70,130,200}  
\usepackage{subcaption}
\usepackage{float}
\usepackage{amsthm}
\newtheorem{proposition}{Proposition}

\lstdefinestyle{pythonstyle}{
  language=Python,
  backgroundcolor=\color{pybggray},
  basicstyle=\ttfamily\small,
  keywordstyle=\color{pykeyword}\bfseries,
  stringstyle=\color{pystring},
  commentstyle=\color{pycomment}\itshape,
  identifierstyle=\color{black},
  emphstyle=\color{pyfunction}\bfseries,               
  emph={function_approximate, graph_approximate, predict, update}, 
  morekeywords={as, with, lambda, yield, from, nonlocal, global, async, await},
  showstringspaces=false,
  breaklines=true,
  tabsize=2,
  frame=none,
  numbers=none,
  xleftmargin=0pt,
  columns=fullflexible,
  aboveskip=0pt,
  belowskip=0pt,
  literate=
    *{0}{{{\color{pynumber}0}}}{1}
     {1}{{{\color{pynumber}1}}}{1}
     {2}{{{\color{pynumber}2}}}{1}
     {3}{{{\color{pynumber}3}}}{1}
     {4}{{{\color{pynumber}4}}}{1}
     {5}{{{\color{pynumber}5}}}{1}
     {6}{{{\color{pynumber}6}}}{1}
     {7}{{{\color{pynumber}7}}}{1}
     {8}{{{\color{pynumber}8}}}{1}
     {9}{{{\color{pynumber}9}}}{1}
     {+}{{{\color{pyoperator}+}}}{1}
     {*}{{{\color{pyoperator}*}}}{1}
     {@}{{{\color{pyoperator}@}}}{1}
     {-}{{{\color{pyoperator}-}}}{1}
     {/}{{{\color{pyoperator}/}}}{1}
     {=}{{{\color{pyoperator}=}}}{1}
}

\newtcolorbox{pythonbox}[1][]{%
  listing only,
  listing engine=listings,
  colback=pybggray,     
  colframe=pybggray,    
  boxrule=0pt,          
  left=0pt,
  right=0pt,
  top=0pt,
  bottom=0pt,
  boxsep=0pt,           
  enhanced,
  sharp corners,
  frame hidden,
  overlay={},           
  valign=center,
  halign=left,
  listing style=pythonstyle,
  #1
}

\title{The Kalman Evolve: Closing the Gap in Kalman Filtering via Interpretable Algorithm Discovery
}

\author{
  Vasileios Saketos \\
  KTH Royal Institute of Technology \\
  Stockholm Sweden\\
  \texttt{saketos@kth.se}
  \And
  Ming Xiao \\
  KTH Royal Institute of Technology \\
  Stockholm Sweden\\
  \texttt{mingx@kth.se}
}

\begin{document}

\maketitle

\begin{abstract}

State estimation is a fundamental problem in control and signal processing, for which the Kalman Filter provides an optimal solution under linear dynamics, Gaussian noise, and known noise covariances. However, these assumptions often fail in realistic sensing settings such as Doppler radar and LiDAR. In these cases, the optimal estimator is inherently nonlinear, which leads to systematic performance degradation. This creates a performance gap that cannot be eliminated by tuning the noise covariance parameters (i.e., the process and measurement noise in the Kalman Filter) alone. To address this limitation, we propose Kalman Evolve, a framework for discovering improved filtering algorithms by jointly optimizing both noise parameters and the update structure. Our approach leverages large language models (LLMs) as a structured prior over program space, enabling the generation of interpretable, non-affine modifications to the classical Kalman filter while preserving its recursive form. We provide analytical results establishing the suboptimality of affine estimators under common nonlinear sensing models, motivating the need for structure-aware updates. Across a range of synthetic and real-world tracking benchmarks, including Doppler radar, LiDAR-based localization, and pedestrian tracking, the discovered algorithms consistently improve over strong baselines such as the Optimized Kalman Filter, achieving up to 12\% reduction in RMSE. These results suggest that optimizing the structure of the Kalman filter, rather than only its parameters, provides a practical and interpretable way to improve state estimation.

\end{abstract}


\section{Introduction}

The Kalman filter ~\citep{10.1115/1.3662552} is a cornerstone of state estimation, widely used due to its optimality and efficiency under linear dynamics and Gaussian noise assumptions. However, modern sensing systems—such as Doppler radar and LiDAR—produce observations that are inherently nonlinear functions of the underlying state. In these settings, the standard Kalman update, which is restricted to affine transformations of the measurement residual, cannot fully capture the structure of the optimal estimator.
Various extensions, such as the Extended Kalman Filter (EKF)~\citep{gelb1974applied} and the Unscented Kalman Filter (UKF)~\citep{Julier1997NewEO}, handle nonlinear dynamics via local linearization and sigma-point approximations, respectively. While preserving the recursive structure of the original algorithm, they assume known system models (with EKF requiring explicit derivatives) and lack optimality guarantees and data-driven adaptability, often degrading under model mismatch. Most data-driven approaches focus on parameter estimation. In particular,~\citep{1624478} estimate the noise covariances Q and R when they are unknown, while more recent work such as the Optimized Kalman Filter (OKF)~\citep{greenberg2023optimization} learns these quantities directly, achieving strong empirical performance. However, these methods remain limited to the affine estimator class and therefore cannot overcome its fundamental limitations.

In parallel, deep learning approaches such as Recurrent Neural Networks (RNNs), Long Short-Term Memory (LSTM) networks, and Gated Recurrent Units (GRUs)~\citep{hochreiter1997long,rumelhart1986learning,cho2014learning} have become standard tools for time-series prediction. Neural Kalman Filters~\citep{revach2022kalmannet,10485649} aim to combine model-based and data-driven approaches to handle partially known dynamics. However, these methods typically rely on black-box function approximation, require substantial training data, lack interpretability, and in some cases may underperform classical Kalman Filters. Notably, prior work (e.g., OKF) demonstrates that optimized Kalman filtering methods consistently outperform neural Kalman Filter approaches across a range of settings, a trend we also observe in our experiments.

Very recently, \cite{saketos2025data} introduced an algorithmic discovery framework based on FunSearch~\citep{RomeraParedes2024} for evolving the Kalman Filter in settings where its assumptions are violated. Their approach considers controlled scenarios in which the Kalman Filter is initially optimal with known parameters, and then systematically relaxes these assumptions, leading to performance degradation. The discovered algorithms demonstrate substantial improvements over the resulting mismatched Kalman Filter in these settings. While this highlights the potential of algorithmic discovery for state estimation, the evaluation is restricted to synthetic scenarios with known system parameters and does not include strong baselines such as adaptive Kalman filters, OKF, or KalmanNet. As a result, it remains unclear whether such approaches yield improvements in realistic settings with unknown system dynamics and parameter uncertainty.

More importantly, what remains missing is a unified approach that jointly addresses both parameter uncertainty and the structural limitations of the Kalman update under model mismatch.
To address this, we move beyond methods that focus solely on either parameter tuning or structure optimization, and instead optimize both the noise parameters and the structure of the Kalman Filter itself, enabling improved performance while preserving the recursive and interpretable form of the Kalman Filter. Unlike prior work on algorithmic discovery, which focuses on controlled settings with known system parameters, we address realistic scenarios in which both the system parameters and the inference structure must be learned from data. We first estimate the process and measurement noise covariances, Q and R, to obtain a well-calibrated and competitive baseline. Building on this, we employ an LLM-assisted evolutionary search to discover improved update structures, enabling improved performance while preserving the recursive and interpretable form of the Kalman Filter. This enables algorithm discovery in realistic settings with unknown dynamics, rather than only in controlled synthetic scenarios.
\begin{itemize}
    \item We show that affine Kalman-style measurement updates are structurally suboptimal under common nonlinear sensing models, including Doppler radar and LiDAR.

    \item We introduce a framework that jointly optimizes noise parameters and update structure, enabling the discovery of interpretable, non-affine extensions of the Kalman Filter while preserving its recursive form. 
    \item We demonstrate that the discovered algorithms consistently outperform strong baselines, including the Optimized Kalman Filter (OKF), across synthetic and real-world benchmarks, achieving up to 12\% reduction in RMSE with comparable computational cost.

    \item We release code and discovered algorithms to support reproducibility and further research in interpretable algorithmic discovery.
\end{itemize}

\section{Preliminaries}

\vspace{-.5em}
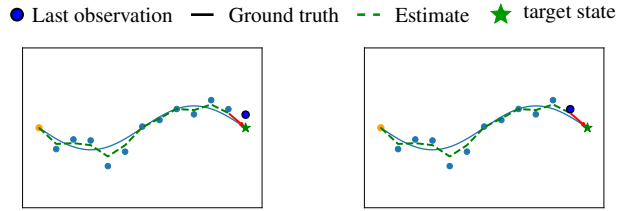
\begin{figure}[h]
\begin{minipage}{0.38\textwidth}
\small

In this work, we address two closely related yet fundamentally distinct tasks, as illustrated in Figure~\ref{fig:tasks}. State estimation (SE) refers to inferring the posterior state given a prior belief and a new observation, whereas next state prediction (NSP) concerns estimating the state from the predictive distribution before the  observation is received.

\end{minipage}
\hfill
\begin{minipage}{0.60\textwidth}
\begin{center}
\footnotesize
\begin{tikzpicture}[baseline=-0.6ex, x=1cm, y=1cm]

\filldraw[fill=blue, draw=black, line width=0.8pt] (0.5,0) circle (2.2pt);
\node[anchor=west] at (0.6,0) {Last observation};

\draw[fill=blue, line width=1.pt] (2.8,0) -- (3.1,0);
\node[anchor=west] at (3.2,0) {Ground truth};
\draw[green!60!black, dashed, line width=1pt] (5.,0) -- (5.3,0);
\node[anchor=west] at (5.4,0) {Estimate};

\node[star, star points=5, star point ratio=2.25,
      fill=green!60!black, draw=green!60!black, inner sep=1.2pt] at (6.9,0) {};
\node[anchor=west] at (7.1,0) {target state};

\end{tikzpicture}%

\end{center}

\vspace{-.5em}

\includegraphics[width=0.45\linewidth]{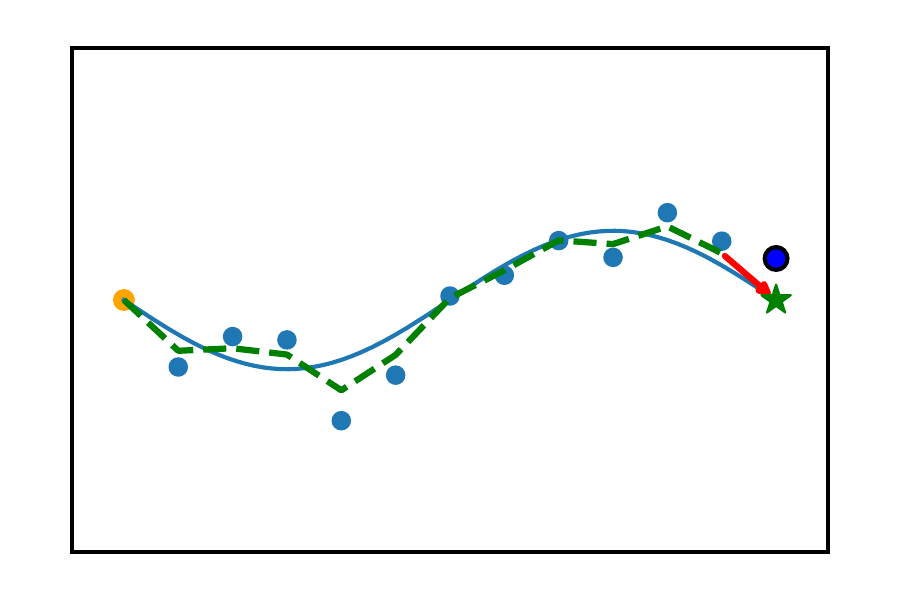}
\hfill
\includegraphics[width=0.45\linewidth]{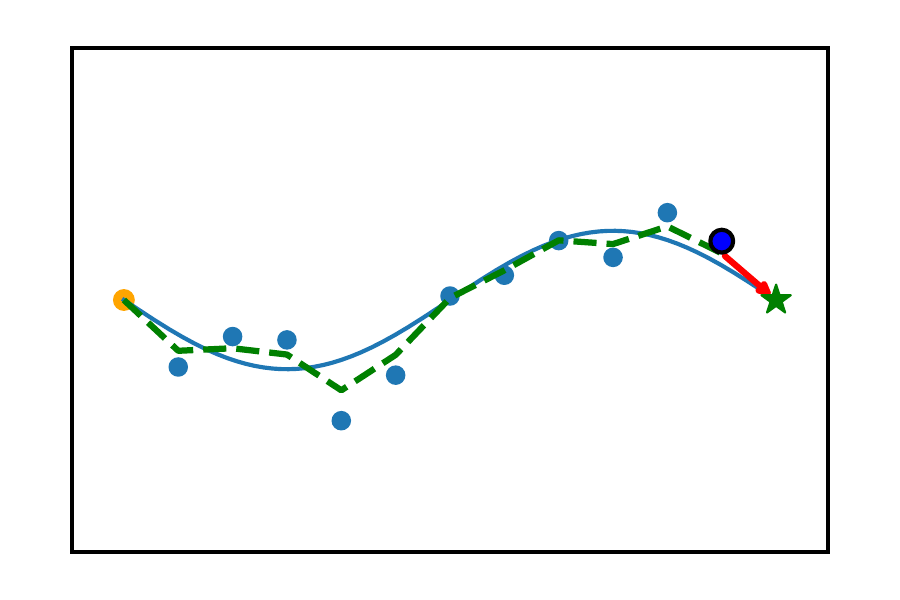}
\vspace{-.5em}

\caption{State estimation (left), Next-state prediction (right)}

\label{fig:tasks}
\end{minipage}
\vspace{-.8em}
\end{figure}

The Kalman Filter, as illustrated in Figure~\ref{fig:Kalman_filter}, is designed for recursive Bayesian state estimation, i.e., computing the posterior distribution $p(x_t \mid z_{1:t})$. This involves first performing the prediction step to obtain the prior distribution $p(x_t \mid z_{1:t-1})$, followed by the update step that incorporates the observation $z_t$. In contrast, next state prediction is obtained by applying the prediction step to the updated state estimate, yielding $p(x_{t} \mid z_{1:t})$. Operationally, this corresponds to first updating the state at time $t$ and then propagating it forward. Thus, while state estimation follows a predict--update sequence, next state prediction can be viewed as an update--predict sequence.

We consider a classical system  evolving according to the following discrete-time linear state equations:
\[ \small
x_t = F x_{t-1} + w_{t-1}, \quad w_{t-1} \sim \mathcal{N}(0, Q) \,\,\,\,\,\,\,\,\,\,\,\,\,\,\,\,\,\,\,\,
z_t = H x_t + v_t, \quad v_t \sim \mathcal{N}(0, R),
\]
where \( F \in \mathbb{R}^{n \times n} \) is the state transition matrix, 
\( H \in \mathbb{R}^{m \times n} \) is the observation matrix, and 
\( Q \) and \( R \) are the covariance matrices of the process and observation noise, respectively. 
The Kalman Filter operates recursively in two main steps: predict and update. 
The predict step propagates the current state and its uncertainty forward in time, 
estimating the system's next state without incorporating new observations. 
Upon receiving the observation \( z_t \), the update step adjusts the prediction using the measurement.

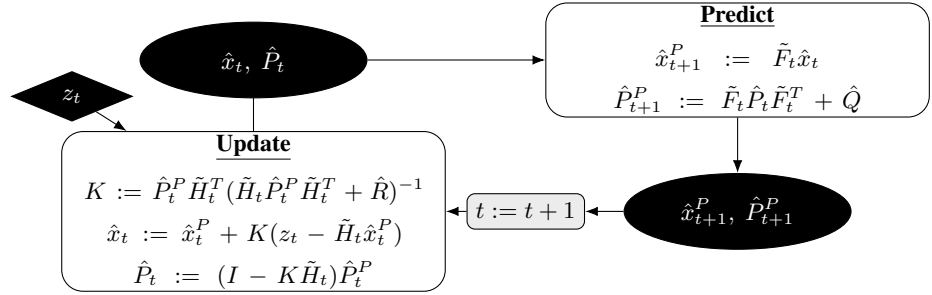
\begin{figure}[h]

\begin{center}
\begin{tikzpicture}[
    scale=0.8,
    >=Latex,
    font=\small,   
    state/.style={
        ellipse, draw, fill=black,
        text=white, align=center,
        minimum width=3.cm, minimum height=1.cm
    },
    block/.style={
        rounded corners=8pt,
        draw,
        text width=5cm,
        align=center,
        inner sep=1pt
    }
]


\node[block] (predict) at (8,2.5) {
\textbf{\underline{Predict}}\\[5pt]
$\hat{x}_{t+1}^P := \tilde{F}_t \hat{x}_t$\\[4pt]
$\hat{P}_{t+1}^P := \tilde{F}_t \hat{P}_t \tilde{F}_t^{T} + \hat{Q}$
};

\node[block] (update) at (-0,0.) {
\textbf{\underline{Update}}\\[5pt]
$K := \hat{P}_t^P \tilde{H}_t^{T}
(\tilde{H}_t \hat{P}_t^P \tilde{H}_t^{T} + \hat{R})^{-1}$\\[4pt]
$\hat{x}_t := \hat{x}_t^P + K(z_t - \tilde{H}_t \hat{x}_t^P)$\\[4pt]
$\hat{P}_t := (I - K\tilde{H}_t)\hat{P}_t^P$
};


\node[state] (xt) at (-0,2.5) {$\hat{x}_t,\ \hat{P}_t$};

\node[state] (xtp) at (8,0) 
{$\hat{x}_{t+1}^P,\ \hat{P}_{t+1}^P$};

\node[draw, rounded corners=3pt,
      fill=gray!15] (time) at (4.5,0) {$t := t+1$};

\node[diamond, draw, fill=black, text=white,
      aspect=2, minimum width=1.6cm] 
      (zt) at (-3,1.9) {$z_t$};


\draw[->] (xt.east) --  (predict.west);

\draw[->] (predict.south) -- (xtp.north);

\draw[->] (xtp.west) -- (time.east);

\draw[->] (time.west) -- (update);

\draw[->] (zt) -- (update);

\draw[->] (update.north) -- ++(0,1) -| (xt.south);
\end{tikzpicture}
\end{center}

\caption{The Kalman Filter.}
\label{fig:Kalman_filter}

\end{figure}

As shown in Figure \ref{fig:Kalman_filter}, the matrix \( K \) is the Kalman gain, determining the relative weighting between the predicted state and the new observation. The posterior estimate \( \hat{x}_{t} \) combines the prior with the measurement residual, while \( P_{t} \) reflects the reduced uncertainty after the update. The Kalman Filter is optimal under the assumptions of linearity and Gaussian noise, providing the minimum mean squared error (MMSE) estimate of the state. However, in real-world settings, the process and observation noise are often unknown and may deviate from Gaussianity. For the case of unknown Gaussian noise covariances \( Q \) and \( R \), \cite{1624478} proposed a least-squares-based algorithm for their estimation. For linear systems with non-Gaussian noise, the OKF framework parameterizes \( Q \) and \( R \) via a Cholesky decomposition~\citep{Pinheiro1996}, enabling their estimation through backpropagation. Although these approaches improve empirical performance, the Kalman Filter remains fundamentally tied to Gaussian assumptions on the prior and observation model. As a result, even when the underlying dynamics are linear, non-Gaussian observation noise generally induces a non-Gaussian posterior distribution, rendering affine estimators such as the Kalman Filter suboptimal. This limitation becomes evident in common sensing models such as Doppler radar and LiDAR.

\begin{Corollary}[Doppler model mismatch]
Consider an idealized Doppler tracking setting. Let \(P\in\mathbb R^3\) denote target position and
\(V\in\mathbb R^3\) denote target velocity. Assume \(P,V\) are jointly Gaussian, isotropic, and
position--velocity correlated:
\(
\mathbb E[P]=\mathbb E[V]=0,\qquad
\operatorname{Cov}(P)=\sigma_P^2I_3,\qquad
\operatorname{Cov}(V,P)=\rho I_3,\qquad \rho\neq0 .
\)
The observation is \(Z=(Y,S)\in\mathbb R^4\), where
\(
Y=P+N\in\mathbb R^3,\qquad
S=u(P)^\top V+\varepsilon\in\mathbb R,\qquad
u(p)=p/\|p\|,
\)
with \(N\sim\mathcal N(0,\sigma_N^2I_3)\), \(\varepsilon\sim\mathcal N(0,\sigma_\varepsilon^2)\), and all
noises independent of each other and of \((P,V)\). Here \(Y\) is a noisy Cartesian position measurement and
\(S\) is a noisy radial-velocity Doppler measurement. Let
\(
\mathcal F_{\rm aff}:=\{f:\mathbb R^4\to\mathbb R^3\mid f(y,s)=a+Ay+bs\}
\)
be the class of affine position estimators from \(Z=(Y,S)\). Then the Bayes estimator
\(g^*(Z):=\mathbb E[P\mid Z]\) is not affine. Consequently,
\(
\inf_{f\in\mathcal F_{\rm aff}}\mathbb E\|P-f(Z)\|^2
>
\mathbb E\|P-g^*(Z)\|^2 .
\label{eq:strict_mse_gap_doppler}
\)
\end{Corollary}
\begin{Corollary}[LiDAR model mismatch]

Similarly, consider a classical LiDAR measurement model. A LiDAR sensor provides a noisy
 observation pair  $( r_z, \theta_z)$.
The quantity $r_z$ denotes the measured distance from the sensor to the target
\( r_z = \|X\| + V_r \),
where $\|X\|$ is the true distance of the target and $V_r$ is the range measurement noise.
The quantity $\theta_z$ denotes the angle of the target relative to the sensor:
\( \theta_z = \Theta + V_\theta \), where $\Theta$ is the true angle of the target and $V_\theta$ is the angular measurement noise. To convert these polar measurements into Cartesian coordinates, we apply the standard transformation.
\(
Z := (Z_x,Z_y)^\top 
= \bigl(r_z\cos\theta_z,\; r_z\sin\theta_z\bigr)^\top .
\)
Let $\mathcal{F}_{\mathrm{aff}}:=\{f:\mathbb{R}^2\to\mathbb{R}^2\mid f(z)=a+Bz,\ a\in\mathbb{R}^2,\ B\in\mathbb{R}^{2\times 2}\}$
denote the affine estimator class (which includes Kalman-style Gaussian measurement updates when applied to Cartesian measurements),
and let $g^*(z):=\mathbb{E}[X\mid Z=z]$ be the Bayes/MMSE estimator. Then there exists an estimator with strictly smaller mean squared
error than any affine estimator:
\( 
\inf_{f\in\mathcal{F}_{\mathrm{aff}}}\mathbb{E}\bigl[\|X-f(Z)\|^2\bigr]
\;>\;
\mathbb{E}\bigl[\|X-g^*(Z)\|^2\bigr].
\label{eq:strict_mse_gap_lidar}
\)
\end{Corollary}


Even in idealized settings with linear dynamics and Gaussian priors, realistic sensing models induce a posterior whose conditional mean $\mathbb{E}[X \mid Z]$ is intrinsically nonlinear. This nonlinearity arises from geometric dependencies in the observation process (e.g., radial or angular measurements), and cannot be captured by any affine estimator. As a result, no parameter choice within the Kalman Filter family can achieve Bayes-optimal performance—the limitation is \emph{structural rather than parametric}. This motivates a shift from parameter estimation to \emph{structure learning}, where the goal is to augment the Kalman update with nonlinear, data-dependent transformations of the innovation while retaining its recursive and interpretable form.

While the previous results establish a structural optimality gap in the update step under model mismatch, many practical applications ultimately require accurate next state prediction. In this setting, it is important to understand how post-update estimation errors and transition model mismatch propagate through the system dynamics. The following lemma characterizes the resulting NSP error.

\begin{lemma}[Next-state prediction error under model mismatch]
Consider \(x_{k+1} = F x_k + w_k\) and \(\hat{x}_{k+1|k} = \tilde{F}\hat{x}_{k|k}\). Let \(e_k := x_k - \hat{x}_{k|k}\) and \(\eta_k := (F-\tilde{F})\hat{x}_{k|k} + w_k\), so that \(x_{k+1}-\hat{x}_{k+1|k} = F e_k + \eta_k\). In our setting the NSP error sequence satisfies \(e_{k+1} = F e_k + \eta_k\), \(\|F^j\|\le C\lambda^j\) for some \(C>0\), \(0<\lambda<1\), and \(\sup_k \mathbb{E}\|\eta_k\|^2 \le M_\eta\). Then, for all \(k\ge 0\),
\[\tiny
\mathbb{E}\|x_{k+1}-\hat{x}_{k+1|k}\|^2
\le
4\|F\|^2 C^2 \lambda^{2k}\,\mathbb{E}\|e_0\|^2
+
\left(\frac{4\|F\|^2 C^2}{(1-\lambda)^2}+2\right)M_\eta,
\]
and consequently
\[ \tiny
\limsup_{k\to\infty}\mathbb{E}\|x_{k+1}-\hat{x}_{k+1|k}\|^2
\le
\left(\frac{4\|F\|^2 C^2}{(1-\lambda)^2}+2\right)M_\eta.
\]
\end{lemma}

The bound shows that the NSP error decomposes into a transient term that vanishes over time $k$ and a persistent term arising from errors propagated through mismatched dynamics and process noise. Since the true dynamics are generally unknown and the optimal correction is not available in closed form, parameter tuning alone may leave a persistent error term. This motivates searching over Kalman-style update rules that can partially compensate for model mismatch under the NSP objective, while preserving the recursive structure of the classical filter. We address this problem using LLM-guided evolutionary search to discover such update rules by directly optimizing long-horizon prediction performance. Appendix~\ref{Appendix:theoretical_analysis} provides formal proofs of Corollaries~1 and~2 as well as Lemma~1.

\section{Kalman Evolve}

\begin{figure}[h]
    \includegraphics[width=\linewidth]{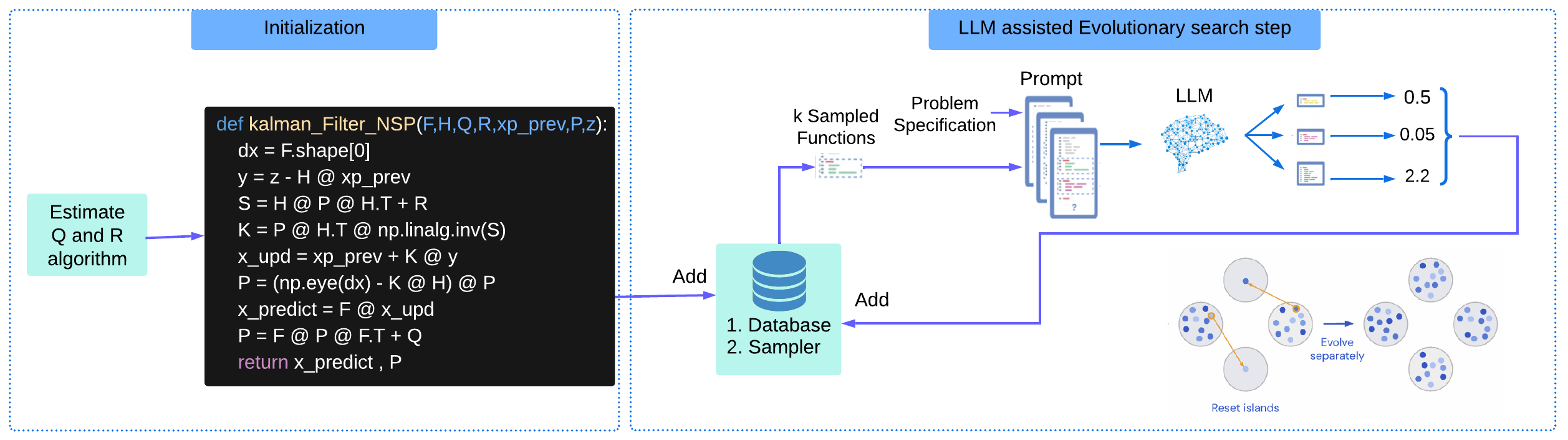}
    \caption{An overview of our framework.}
    \label{fig:Framework}
\end{figure}

We introduce Kalman Evolve, a two-stage framework that jointly optimizes noise parameters and update structure for state estimation. In the first stage, we estimate the noise covariance matrices Q and R to obtain a calibrated baseline using either least-squares estimation or the Optimized Kalman Filter (OKF). In the second stage, we perform LLM-guided evolutionary search, where the LLM acts as a structured prior over the space of interpretable algorithmic transformations, enabling the discovery of non-affine Kalman update rules while preserving the recursive filtering framework. This formulation unifies parameter calibration and structure learning, allowing systematic improvements over classical and optimized Kalman filtering approaches.


These estimated parameters, together with the initial algorithm, are added to the database used by the LLM-guided evolutionary search. The LLM-assisted evolutionary search starts by sampling $k$ algorithms from the database. The problem specification, which formally defines the problem we try to address, is concatenated with the two candidate algorithms. This combined representation constitutes the input prompt to the LLM. A prompt example is provided in Appendix~\ref{Appendix:Prompt}. The prompt is subsequently forwarded to the LLM, which in our case is DeepSeek 16B \citep{guo2025deepseek}. The LLM processes the prompt and produces mutations and
combinations of the two input functions. After generation, each candidate solution is evaluated based on a fitness/loss function. The functions and fitness values are used to update the database that maintains the top {\it N} highest-scoring candidates. Our LLM-assisted evolutionary search is strongly inspired by AlphaEvolve ~\citep{novikov2025alphaevolve}. Since the AlphaEvolve implementation is not publicly available, we implement this procedure following the published description. As demonstrated in Figure \ref{fig:Framework}, the ES is distributed across four independent islands, each running on a separate GPU with its own database. Each island performs 10 iterations of sampling, mutation, and database updates per cycle. After each cycle, the islands with the weakest top-performing candidates are reinitialized using the best candidate from the highest-performing island. This process is repeated for 20 cycles.

\section{Kalman Evolve versus Optimized Kalman Filter}
In this section, we evaluate Kalman Evolve across a wide range of challenging scenarios, comparing it with strong classical baselines, including the Kalman Filter and the Optimized Kalman Filter. All methods are assessed on both state estimation and next state prediction tasks. Across these benchmarks, Kalman Evolve improves over optimized baselines while maintaining interpretability and low inference overhead.
In Appendix~\ref{Appendix:Additional_experiments}, we further compare against deep learning approaches such as KalmanNet and LSTM, which generally exhibit weaker performance.


\subsection{The Doppler radar problem}
\begin{figure}[h]
    \centering

    \begin{subfigure}[t]{0.19\textwidth}
        \centering
        \includegraphics[height=2.5cm,width=\linewidth]{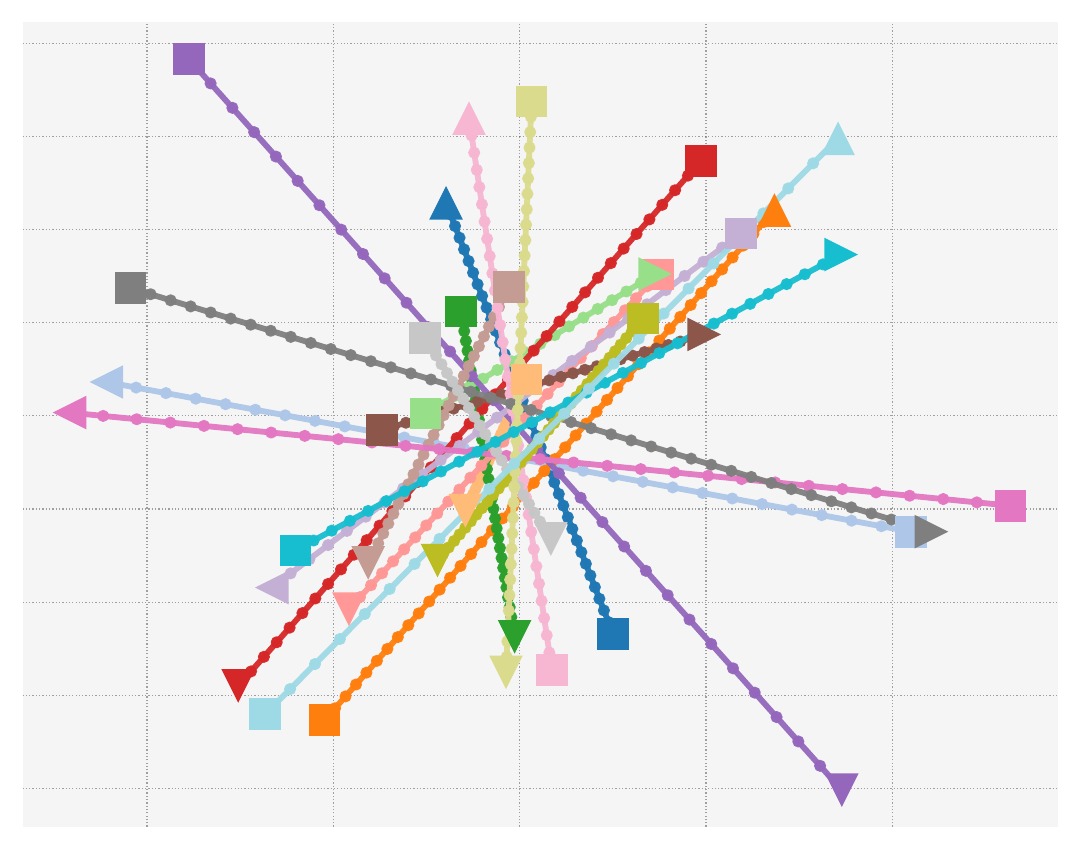}
        \caption{Toy}
    \end{subfigure}
    \hfill
    \begin{subfigure}[t]{0.19\textwidth}
        \centering
        \includegraphics[height=2.5cm,width=\linewidth]{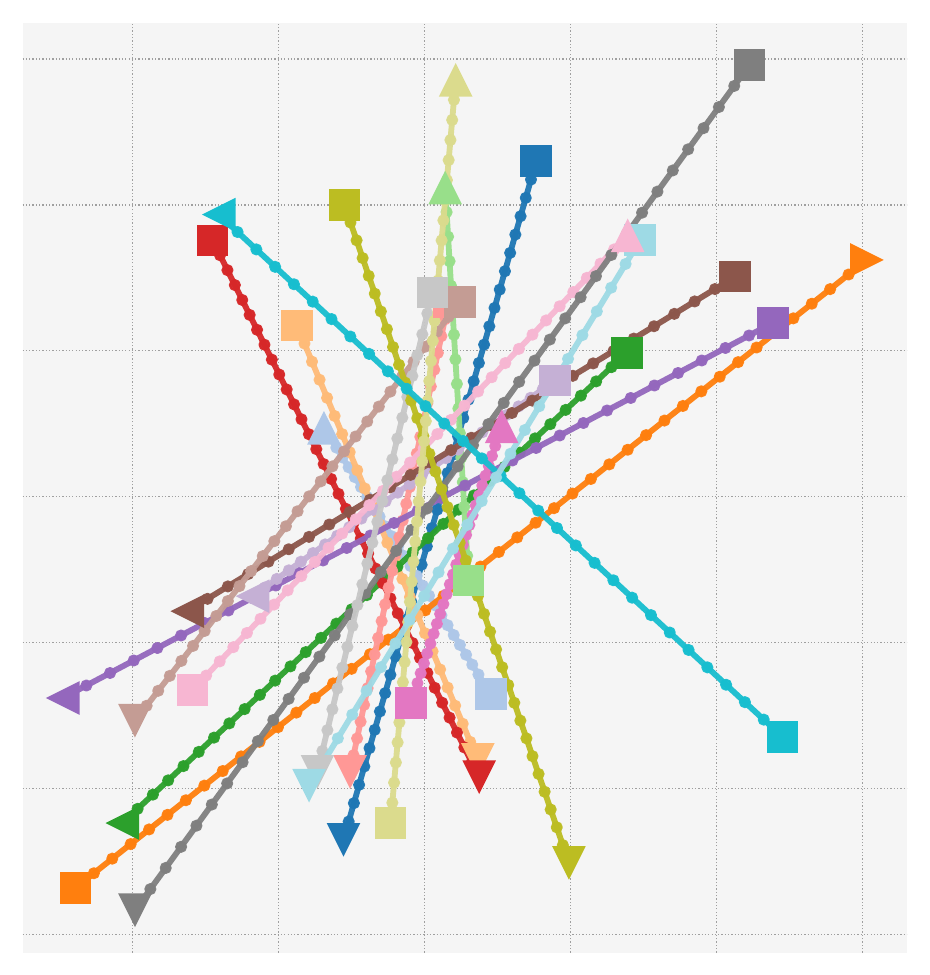}
        \caption{Close}
    \end{subfigure}
    \hfill
    \begin{subfigure}[t]{0.19\textwidth}
        \centering
        \includegraphics[height=2.5cm,width=\linewidth]{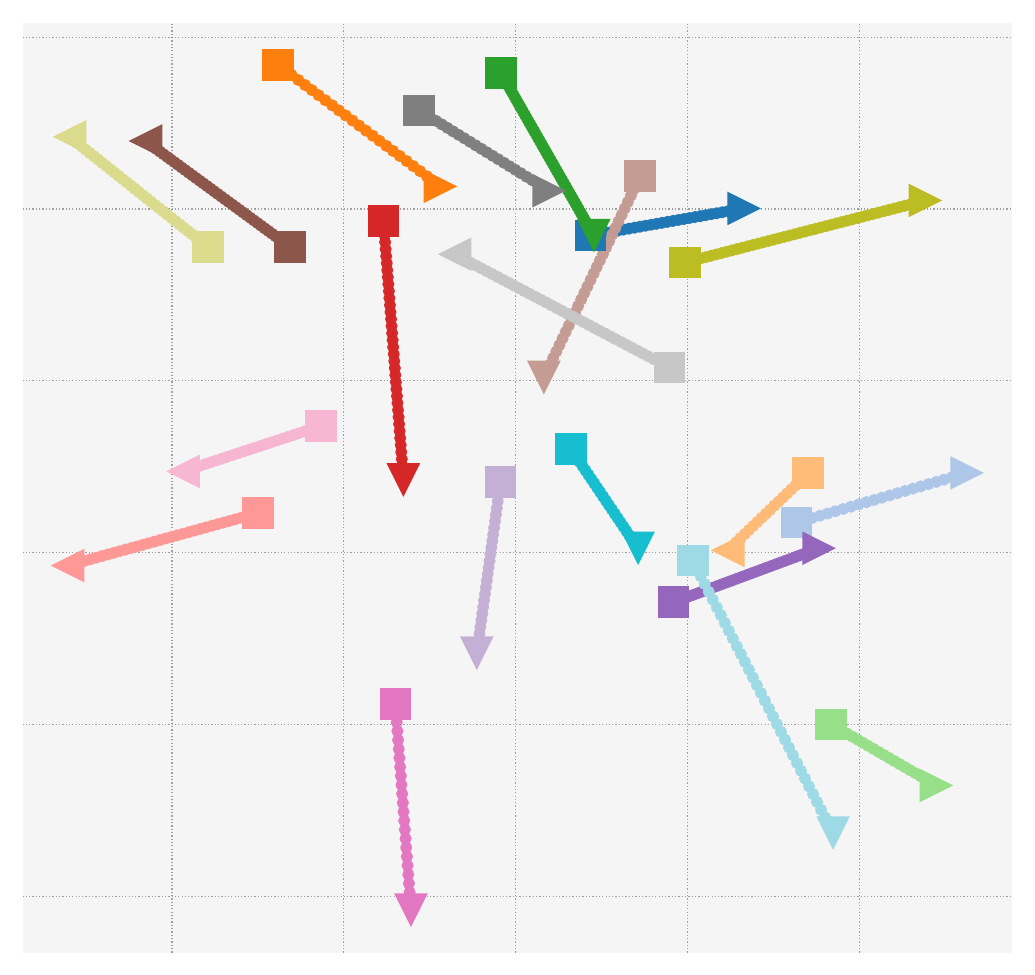}
        \caption{Const\_v}
    \end{subfigure}
    \hfill
    \begin{subfigure}[t]{0.19\textwidth}
        \centering
        \includegraphics[height=2.5cm,width=\linewidth]{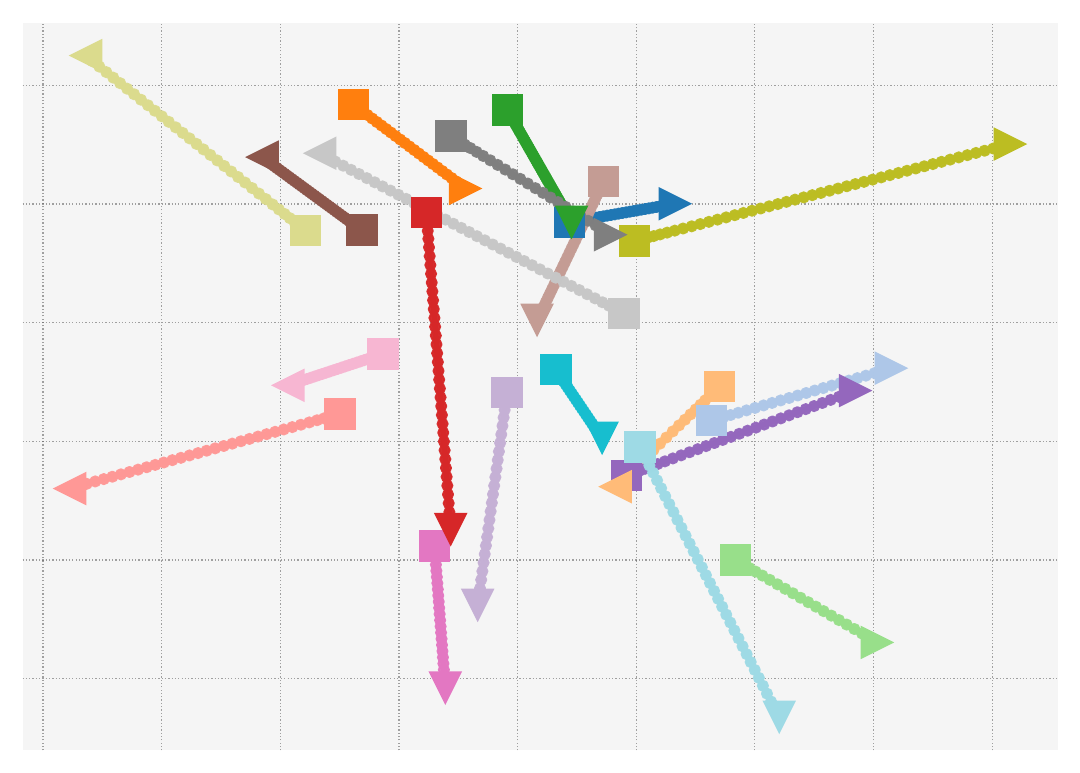}
        \caption{Const\_a}
    \end{subfigure}
    \hfill
    \begin{subfigure}[t]{0.19\textwidth}
        \centering
        \includegraphics[height=2.5cm,width=\linewidth]{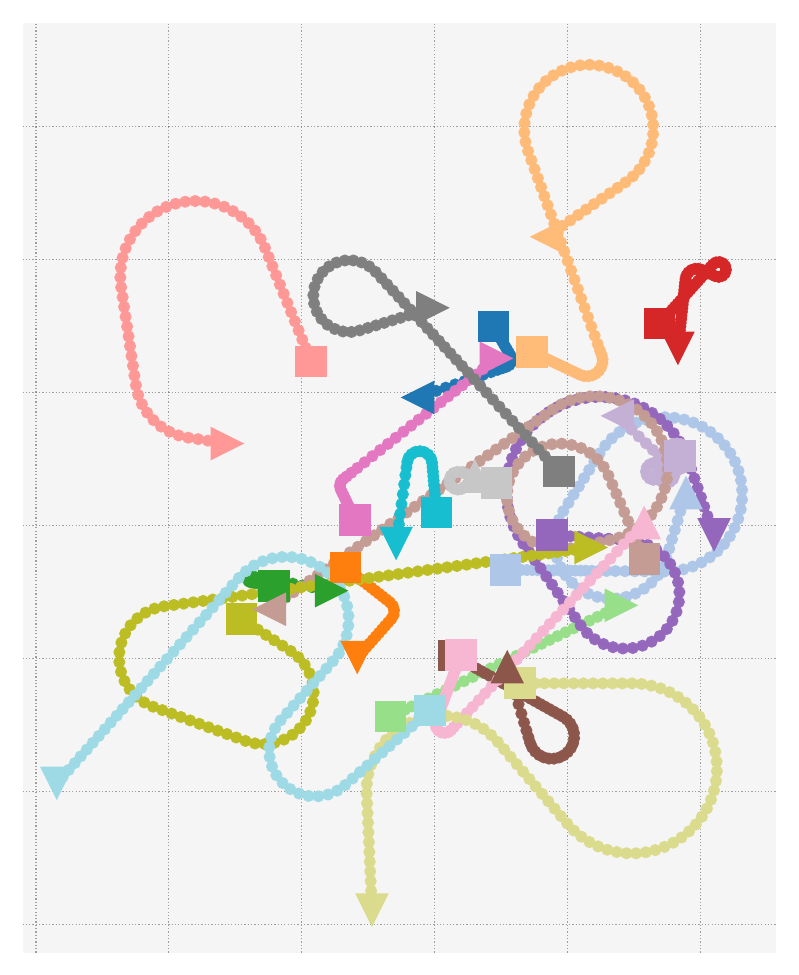}
        \caption{Free}
    \end{subfigure}

    \caption{Trajectory examples under different dynamics.}
\end{figure}
We consider a variant of the classic Doppler radar problem \citep{Pell_1989, Mitra}, where multiple target trajectories are tracked in a homogeneous 3D space, given regular observations from a Doppler radar. The state \(
X = (x_x, x_y, x_z, x_{vx}, x_{vy}, x_{vz})^\top \in \mathbb{R}^6
\) represents the 3D position and velocity of a target. The goal is to minimize the RMSE over the three position coordinates. While the true dynamics $F$ are unknown to the Kalman Filter (KF), a constant-velocity approximation $\tilde{F}$ is used. Observations $Z \in \mathbb{R}^4$ consist of the target’s position in spherical coordinates (range, azimuth, elevation) along with the radial velocity (Doppler measurement), corrupted by additive i.i.d. Gaussian noise. After transformation to Cartesian coordinates, the observation model can be written as:
\[
\tiny\tilde{F} =
\begin{pmatrix}
1 &  &   & 1 &   &   \\
  & 1 &   &  & 1  &   \\
  &   & 1 &   &  &  1 \\
  &   &   & 1 &   &   \\
  &   &   &   & 1 &   \\
  &   &   &   &   & 1
\end{pmatrix}
, \,\,\,\,\,\,
H = H(X) =
\begin{pmatrix}
1 &   &   &   &   &   \\
  & 1 &   &   &   &   \\
  &   & 1 &   &   &   \\
 &  &  & \frac{x_x}{r} & \frac{x_y}{r} & \frac{x_z}{r} 
\end{pmatrix}
\]
where \( \tiny
r = \sqrt{x_x^2 + x_y^2 + x_z^2}.
\) Since $H = H(X)$ relies on the unknown location $(x_x, x_y, x_z)$, we instead substitute $\tilde{H} := H(Z)$ in the KF update step in Figure \ref{fig:Kalman_filter}. 

We evaluate our method on a suite of benchmarks spanning a range of motion complexities, from simplified settings to more realistic scenarios. These benchmarks are designed to isolate key properties such as anisotropy, noise structure, and motion dynamics, enabling controlled comparisons across different levels of difficulty. A detailed description of each benchmark and its specific properties is provided in Appendix~\ref{Appendix:Doppler}.

\begin{table}[H]
\centering
\begin{tabular}{|l|c|c|c|c|}
\hline
 & KF & OKF & KE  & Observations \\
\hline
Toy      & 109.24 & 78.31  &  \textbf{78.03} & 173.10 \\
Close    & 19.71  & 19.73  & \textbf{18.62} & 44.19  \\
Const\_v & 85.81  & 83.5 & \textbf{73.95} & 198.41 \\
Const\_a & 95.57 & 91.78 & \textbf{81.63} & 216.58 \\
Free     & 101.90 & 95.72 &  \textbf{84.23} & 186.42 \\
\hline
\end{tabular}
\caption{Performance comparison of filtering methods across benchmarks in terms of RMSE. Lower values indicate better performance; the best result for each benchmark is highlighted in bold.}\label{tab:kalman_results_Dopler}
\end{table}
Table~\ref{tab:kalman_results_Dopler} shows that KE consistently outperforms OKF across all benchmarks, with gains increasing as task complexity grows. In the Toy setting, the improvement is marginal (0.4\%), while in Close, KE achieves a more noticeable gain of 5.6\%. For more structured dynamics, the improvements become substantial: KE reduces the error by 11.4\% in Const\_v and 11.1\% in Const\_a relative to OKF. In the most challenging Free setting, the gain reaches 12.0\%. Overall, these results indicate modest improvements in simple settings and consistent gains of around 10--12\% in more complex regimes. These results suggest that KE preserves the strengths of OKF while adding useful modeling capacity when the dynamics deviate from linear or stationary assumptions. The advantage is most pronounced in scenarios involving acceleration and nonlinear motion. In Appendix~\ref{Appendix:Doppler}, we report performance in terms of NSE, evaluate additional baselines, provide a comprehensive out-of-distribution generalization study, and analyze how performance scales with the number of training iterations.
\subsection{LiDAR-based State Estimation }
In line with the OKF paper, we consider the problem of state estimation in self-driving scenarios based on LiDAR measurements with respect to known landmarks \citep{MOREIRA2020512}. We simulate driving trajectories composed of multiple segments with varying accelerations and turning radii. The state consists of the vehicle’s 2D position and velocity, while $\tilde{F}$ is modeled using a constant-velocity assumption. In addition, we experiment with the NCLT dataset \citep{10.1177/0278364915614638}. For both datasets, the observations (both true $H$ and modeled $\tilde{H}$) correspond to the vehicle’s position, corrupted by additive i.i.d.\ Gaussian noise in polar coordinates. This results in the following model:
\small
\[
\tiny\tilde{F} =
\begin{pmatrix}
1 & 0 & 1 & 0 \\
0 & 1 & 0 & 1 \\
0 & 0 & 1 & 0 \\
0 & 0 & 0 & 1
\end{pmatrix}, \,\,\,\,\,
\tilde{H} = H =
\begin{pmatrix}
1 & 0 & 0 & 0 \\
0 & 1 & 0 & 0
\end{pmatrix}.
\]

Table~\ref{tab:kalman_results_LiDAR} reports the performance of all methods. On the synthetic dataset, the model discovered by Kalman Evolve improves over OKF by approximately 5\% in state estimation and 6\% in next state prediction. The NCLT dataset contains 27 trajectories (21/3/3 train/validation/test) and is relatively small and noisy. Nevertheless, KE improves over OKF by approximately 2\% in SE and 3\% in NSP, suggesting that the discovered update transfers beyond the synthetic setting. Notably, OKF already provides substantial improvements over KF, highlighting the benefit of parameter tuning, while KE yields additional consistent gains.\begin{table}[H]
\centering
\begin{tabular}{|l|c|c|c|c|}
\hline
 & KF & OKF & KE  & Observations \\
\hline
synthetic SE      & 12.70 & 11.16&  \textbf{10.52} & 18.15  \\
synthetic NSP    & 27.80  & 23.94  & \textbf{22.72} & 76.53  \\
NCLT SE & 26.3  &  23.0 & \textbf{22.6} & 47.2 \\
NCLT NSP  & 26.8 & 23.1 & \textbf{22.5}  & 50 \\
\hline
\end{tabular}
\caption{Performance comparison of filtering methods across LiDAR benchmarks in terms  of RMSE.}\label{tab:kalman_results_LiDAR}
\end{table}
Additional details on the datasets and results—including test-time analysis, statistical evaluation, and performance as a function of search budget—are provided in Appendices~\ref{Appendix:LiDAR} and~\ref{Appendix:NCLT}.
\newpage
\subsection{Pedestrian Tracking}

\begin{figure}[h]
\begin{minipage}{0.38\textwidth}
\small

The MOT20 dataset~\citep{dendorfer2020mot20} contains videos of real-world targets (primarily pedestrians, as shown in Figure~\ref{fig:mot20figure}), along with their ground-truth location and size in each frame. Since the dataset does not provide noisy observations, state estimation is not applicable; instead, we focus on next state prediction. The objective is to predict the target location in the next frame. The state space includes the 2D position, size, and velocity, while the observations consist only of position and size. The underlying dynamics $F$ are unknown, and a standard constant-velocity model is used for $\tilde{F}$. This leads to the following model:\[
\tilde{F} =
\begin{pmatrix}
1 & 0 & 0 & 0 & 1 & 0 \\
0 & 1 & 0 & 0 & 0 & 1 \\
0 & 0 & 1 & 0 & 0 & 0 \\
0 & 0 & 0 & 1 & 0 & 0 \\
0 & 0 & 0 & 0 & 1 & 0 \\
0 & 0 & 0 & 0 & 0 & 1 \\
\end{pmatrix},
\]
\end{minipage}
\hfill
\begin{minipage}{0.55\textwidth}
\begin{center}
\footnotesize
    \vspace{-65pt}

\textcolor{KEcolor}{\rule{1em}{.5em}}\, KE \quad
\textcolor{OKFcolor}{\rule{1em}{.5em}}\, OKF \quad
\textcolor{KFcolor}{\rule{1em}{.5em}}\, KF \quad
\end{center}
\includegraphics[width=0.98\linewidth]{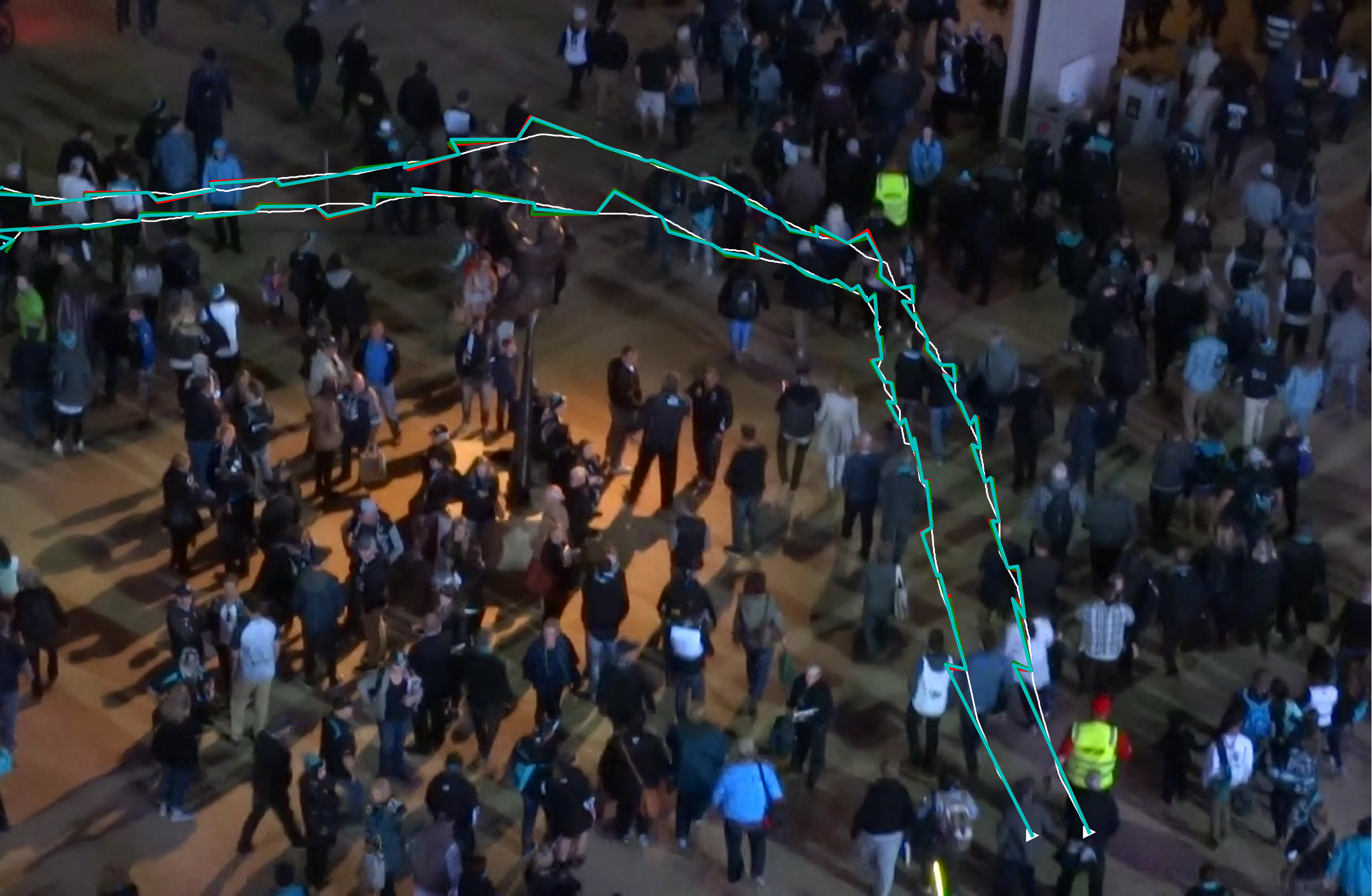}
\caption{Visualization of two trajectories in the first frame of the MOT20 test video with predictions from KF, OKF, and KE.}
\label{fig:mot20figure}
\end{minipage}

\label{fig:performance}
\end{figure}

\begin{figure}[h]
\noindent
\begin{minipage}[t]{0.38\textwidth}
\vspace{-17pt}

\small
\[
\tilde{H} = H =
\begin{pmatrix}
1 & 0 & 0 & 0 & 0 & 0 \\
0 & 1 & 0 & 0 & 0 & 0 \\
0 & 0 & 1 & 0 & 0 & 0\\
0 & 0 & 0 & 1 & 0 & 0\\
\end{pmatrix}.
\]

Our observation model satisfies $\tilde{H} = H$ and is \emph{linear}, i.e., $H$ is independent of $X$. In our setup, the first three videos (1117 trajectories) are used for training, and the final video (1208 trajectories) for testing. 
Figure~\ref{fig:mot20figure} illustrates two example pedestrian trajectories, where small but consistent differences between methods can be observed. As shown in Figure~\ref{fig:Mot_performance}, KE reduces the test RMSE relative to OKF by 8\%, with high statistical significance.


\end{minipage}
\hfill
\begin{minipage}[t]{0.55\textwidth}
\centering
\footnotesize
\vspace{-42pt}
\vspace{0pt}

\begin{minipage}[t]{0.48\linewidth}
    \centering
    \includegraphics[width=\linewidth]{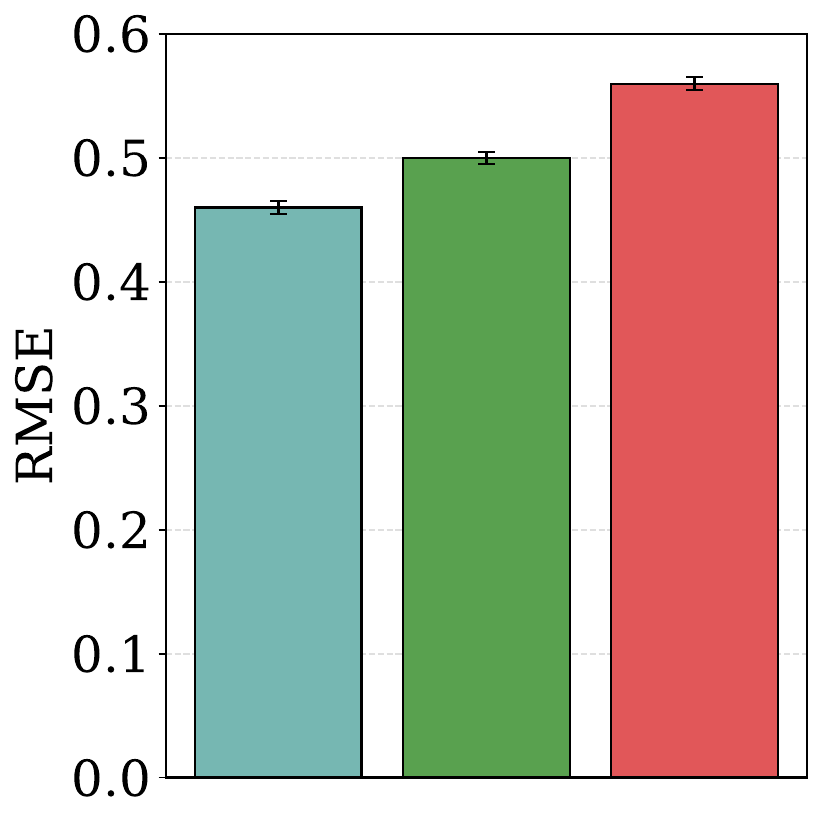}
\end{minipage}
\hfill
\begin{minipage}[t]{0.5\linewidth}
    \centering
    \vspace{-125pt}
    \hspace*{3mm}{\small\bfseries\kern0.3pt\rmfamily z-value (KE-OKF) 13.4\,[\%]}

    \includegraphics[width=\linewidth]{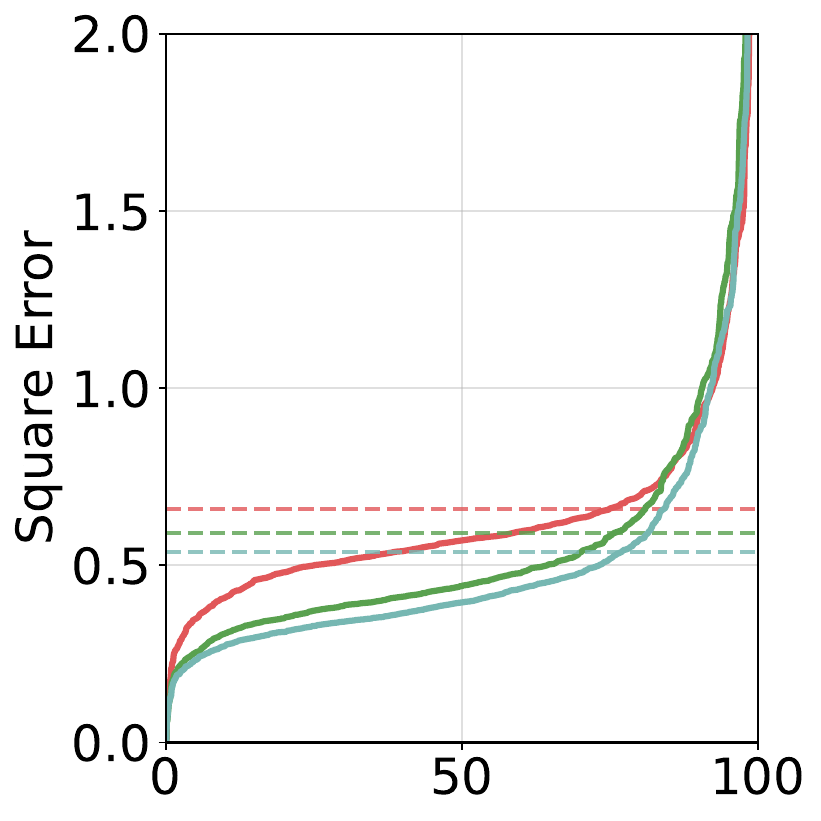}
    \vspace{-12pt}
    
    \hspace*{8mm}{\small\bfseries\kern0.3pt\rmfamily quantile\,[\%]}
    
\end{minipage}
\caption{Performance for next state prediction (left plot) z-test (right plot). The dashed lines correspond to MSE. Both $z$-values correspond to 
$p$-value $< 10^{-6}$. Each $z$-value is calculated over $N$ test 
trajectories as follows:
$z = \frac{\mathrm{mean}(\{\Delta_i\})}{\mathrm{std}(\{\Delta_i\})} 
\sqrt{N},
\qquad
\text{where } 
\Delta_i = \mathrm{err}_i(KF)^2 - \mathrm{err}_i(OKF)^2$
is the square-error difference on trajectory $i$, $1 \le i \le N$.}
\label{fig:Mot_performance}
\end{minipage}
\end{figure}

\begin{figure}[h]
\begin{minipage}{0.38\textwidth}
\small
\vspace{-30pt}

Figure~\ref{fig:Performance_per_samples_mot} illustrates how performance evolves with the number of training samples. KE consistently achieves the lowest RMSE and remains largely insensitive to dataset size. In contrast, OKF benefits significantly from additional training data, showing a sharp improvement at small sample sizes followed by a plateau, while KF performs worst overall and is largely unaffected by increasing data. Overall, KE demonstrates superior and stable performance across all settings, outperforming both KF and OKF regardless of the amount of training data. In Appendix~\ref{Appendix:mot_20_prerformance}, we report additional performance results along with the corresponding test times.
\vspace{-27pt}

\end{minipage}
\hfill
\begin{minipage}{0.55\textwidth}
\begin{center}
\footnotesize
\vspace{-17pt}
\textcolor{KEcolor}{\rule{1em}{.5em}}\, KE \quad
\textcolor{OKFcolor}{\rule{1em}{.5em}}\, OKF \quad
\textcolor{KFcolor}{\rule{1em}{.5em}}\, KF \quad
\includegraphics[width=.9\linewidth, height=0.21\textheight]{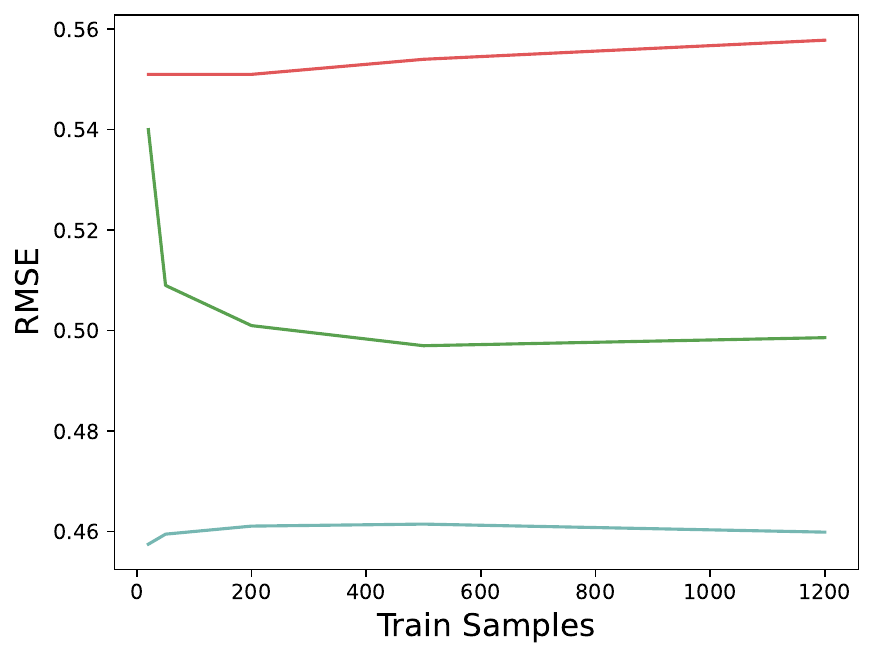}
\caption{Performance as a function of samples}
\label{fig:Performance_per_samples_mot}
\end{center}
\vspace{-20pt}
\end{minipage}
\label{fig:performance}
\end{figure}

\newpage
\section{Discussion}
In this section, we discuss the properties of Kalman Evolve, including its performance, interpretability, limitations, and computational complexity, in comparison to both classical and learning-based Kalman filtering methods.

\textbf{Deep learning based Kalman Filters:} In Appendix \ref{Appendix:Additional_experiments}, we benchmark KalmanNet against all considered approaches under extensive hyperparameter tuning. We find that it exhibits inconsistent performance while incurring significantly higher computational cost than the classical Kalman Filter and often fails to outperform simpler models such as LSTM. Our results are consistent with prior work \citep{NEURIPS2024_f1cf02ce}. The OKF study evaluated a linear Kalman Network, which frequently outperformed the classical Kalman Filter. Nevertheless, it remained inferior to OKF and, in certain cases, even to the Kalman Filter. Table~\ref{tab:full_LiDAR_table_main} reports the performance of all methods on the LiDAR synthetic dataset, a setting in which deep learning approaches perform relatively well compared to other benchmarks. Nevertheless, their performance remains inferior to that of the Kalman Filter.

\begin{table}[h]
\centering
\setlength{\tabcolsep}{4.3pt}   
\begin{tabular}{|c|c|c|c|c|c|c|c}
\hline
 & \begin{tabular}{c}KF\end{tabular}
 & \begin{tabular}{c}OKF\end{tabular}
 & \begin{tabular}{c}KE\end{tabular}
 & LSTM
 & \begin{tabular}{c}KN\end{tabular}
 & OBS \\
\hline
SE 
& 12.70 $\pm$ 0.1
& 11.16 $\pm$ 0.2
& {\bf 10.52}  $\pm$ 0.2
& 20.22 $\pm$ 0.1£
& 14.81 $\pm$ 0.1
& 18.15 $\pm$ 0.3\\
\hline

NSP 
& 27.80 $\pm$ 0.2
& 23.94 $\pm$ 0.2
& \textbf{22.40} $\pm$ 0.2
& 31.52  $\pm$ 0.2
& 29.62  $\pm$ 0.2
& 76.53 $\pm$ 0.4\\

\hline
\end{tabular}

\caption{Root mean squared error performance of different methods $\pm$ stderr for LiDAR tracking }
\label{tab:full_LiDAR_table_main}
\end{table}

\textbf{Interpretability of our algorithms:} 
The learned algorithm preserves the standard Kalman filtering operations, including the innovation $y = z - Hx^-$, which measures the residual between observation and prediction, and the Kalman gain $K = P H^T (H P H^T + R)^{-1}$, which determines the relative weighting of the measurement.  The state update is then modified through a scalar gating mechanism, yielding $x^+ = x^- + \text{gate} \cdot K y$, which adaptively modulates the magnitude of the correction based on the residual. Similarly, the covariance update becomes $P^+ = (I - \text{gate} \cdot K H)P$, preserving its structure while adapting confidence. In the prediction step, the standard propagation $x^- = F x^+$ is unchanged, while the covariance prediction is modified by scaling the process noise, $P^- = F P^+ F^T + \alpha Q$, enabling adaptive uncertainty estimation. These modifications preserve the semantic roles of the classical filter while introducing simple, interpretable mechanisms for adaptive correction and uncertainty modulation. In Appendix~\ref{Appendix:Doppler}, we provide more examples of the discovered algorithms, showing that they consistently follow a simple structure based on nonlinear residual transformations and adaptive rescaling.

\begin{figure}[h]
\centering
\noindent
\begin{minipage}[t]{0.70\textwidth}
  \begin{pythonbox}[height=4.9cm]
  \begin{lstlisting}[style=pythonstyle]
def mot_algorithm(xp_prev, F, P, Q, z, R, H):
    y = z - H @ xp_prev
    S = H @ P @ H.T + R + 1e-10 * np.eye(H.shape[0])
    inv_S = np.linalg.inv(S)
    K = (P @ H.T) @ inv_S
    gate_input = np.mean(y**4) + np.std(y**2)**2
    gate = 0.5 * (1 + np.tanh(gate_input))
    x_upd = xp_prev + gate * K @ y
    P_upd = (np.eye(F.shape[0]) - gate * K @ H) @ P
    scaling_factor = np.tanh(np.mean(y**2) + np.std(y**2))
    P_next = F @ P_upd @ F.T + scaling_factor * Q
    xp_next = F @ x_upd
    return xp_next, P_next
  \end{lstlisting}
  \end{pythonbox}
\end{minipage}%
\caption{{ Discovered algorithm for MOT benchmark and next state prediction task.}}
\label{fig:discovered_function_mot}

\end{figure}

\textbf{Limitations: }Despite the observed performance gains, the discovered algorithms are not theoretically optimal, and further improvements remain possible. Moreover, although they are explicitly represented as text and thus potentially interpretable, the LLM-driven discovery process is not fully controllable and may yield solutions that are difficult to interpret.
 
\textbf{Discovery and inference complexity  of our algorithms: } The primary computational cost of our approach lies in the discovery phase, where algorithms are learned via LLM-assisted evolutionary search. Once discovered, the resulting algorithms have an inference cost comparable to the standard Kalman Filter, as they retain the same computational structure with only minor additional operations. Empirically, the LLM-assisted evolutionary search often converges within a few cycles in some cases, while in others it continues to improve throughout the entire search process, suggesting that additional cycles can further enhance performance. A full discovery run takes approximately 2--3 days on 4 H100 GPUs without vLLM, and 4--6 hours with vLLM, depending on the training set size.
\section{Conclusion  and future work}

We introduce Kalman Evolve, a unified machine learning framework for evolving Kalman Filters, which consistently achieves state-of-the-art performance on state estimation tasks. The proposed approach preserves the interpretable structure of the classical Kalman Filter while augmenting it with simple, data-driven mechanisms that improve robustness to noise, outliers, and model mismatch. Across a wide range of benchmarks, the discovered algorithms demonstrate strong generalization and data efficiency, with minimal additional inference cost compared to standard filtering methods. An important aspect of our framework is the separation between discovery and deployment: while the search process is computationally intensive, the resulting algorithms retain the efficiency and simplicity of classical filters. This makes them practical for real-time applications while benefiting from modern data-driven optimization. For future work, we plan to explore broader applications of Kalman filtering, including more complex real-world scenarios such as brain-machine interfaces \citep{NEURIPS2024_f1cf02ce}. Moreover, we aim to investigate the role of prior knowledge in guiding the LLM-assisted evolutionary search, in particular how structured prompts that include prior knowledge about the problem affect convergence, stability, and performance. Finally, we will extend the proposed methodology beyond filtering, applying similar principles to the discovery of wavelet representations and other structured signal processing algorithms. Overall, our results suggest that combining classical model-based structure with LLM-driven discovery provides a promising direction for developing interpretable and high-performance algorithms.
\section{Related work}
The advent of Transformers \citep{DBLP:journals/corr/VaswaniSPUJGKP17} has spurred the development of numerous Large Language Models (LLMs), such as BERT \citep{devlin-etal-2019-bert} and GPT \citep{openai2024gpt4}, revolutionizing the processing and analysis of written language. Recently, reasoning capabilities have become a central focus in the development of large language models (LLMs). In this context, DeepSeek ~\citep{guo2025deepseek} has emerged as one of the first open-source models to rival the performance of closed-source alternatives ~\citep{comanici2025gemini} by combining large-scale pretraining with reinforcement learning techniques to enhance logical inference.
As such, large language models (LLMs) typically consist of billions of parameters, posing significant computational and memory challenges. To address these limitations, techniques such as QLoRA~\citep{dettmers2023qlora} have been proposed to enable memory-efficient fine-tuning by combining low-bit (e.g., 4-bit) quantization with parameter-efficient adaptation. In parallel, efficient inference frameworks such as vLLM ~\citep{kwon2023efficient} have been developed to improve deployment efficiency by optimizing memory usage and throughput, thereby enabling scalable serving of large models in real-world applications.\\
Genetic Programming (GP) is a subfield of computer science inspired by the principles of natural evolution, in which computer programs are iteratively evolved using mechanisms such as selection, mutation, and crossover \citep{koza1992genetic,10.5555/1796422}. 
Frameworks such as Cartesian Genetic Programming (CGP) \citep{10.1145/2739482.2756571} employ a structured representation of the solution space using fixed-length directed acyclic graphs. 
GP algorithms maintain a population of candidate solutions that are iteratively refined based on random sampling and mutations as well as the performance on a pre-defined fitness function.  Despite their success across a wide range of applications, the design of mutation operators and the selection of building blocks (i.e., graph nodes) remain critical factors influencing the performance of the discovered algorithms. In particular, suboptimal choices in these components can substantially degrade the quality and efficiency of the resulting solutions. Moreover, these building blocks are typically defined at the beginning of the process and remain fixed throughout the evolutionary procedure, which may limit the adaptability of the approach and constrain the exploration of potentially more suitable representations during the search. Recently, LLMs have emerged as a promising approach for addressing this limitation.\\
 FunSearch~\citep{RomeraParedes2024} introduces a hybrid framework combining large language models with evolutionary search, to discover high-quality program variants guided by task-specific evaluation, enabling the discovery of state-of-the-art heuristics for challenging combinatorial problems such as bin packing~\citep{10.5555/2183}, the cap set problem~\citep{capset}, and the admissible sets problem~\citep{Tao_Vu_2006}. FunBO~\citep{Aglietti2024FunBODA} extended FunSearch to automatically discover acquisition functions for Bayesian optimization, achieving strong generalization and competitive or superior performance to standard and problem-specific methods.  The successor to FunSearch, AlphaEvolve~\citep{novikov2025alphaevolve}, enabled the discovery of a novel method for \(4 \times 4\) complex matrix multiplication requiring only 48 scalar multiplications, marking the first improvement over the Strassen's algorithm ~\citep{strassen1969gaussian} in over 60 years. Despite the importance and significance of algorithmic discovery through LLMs, the authors did not open-source their code. Consequently our implementation is inspired by AlphaEvolve but not equivalent to it.



\bibliography{tmlr}
\bibliographystyle{apalike}


\appendix

\section{Theoretical Analysis}
\label{Appendix:theoretical_analysis}
 In this section, we provide a formal proof of Corollaries 1 and 2 as well as Lemma 1.

\subsection{Proof of Corollary 1}

\paragraph{Corollary 1 (Position non-affinity under idealized Doppler observations).}
Let \(P,V\in\mathbb R^3\) denote the position and velocity of a target. Assume \(P,V\) are jointly Gaussian, isotropic, and position--velocity correlated:
\[
\mathbb E[P]=\mathbb E[V]=0,\qquad
\operatorname{Cov}(P)=\sigma_P^2I_3,\qquad
\operatorname{Cov}(V)=\sigma_V^2I_3,\qquad
\operatorname{Cov}(V,P)=\rho I_3,\qquad \rho\neq0 .
\]
The observation is \(Z=(Y,S)\in\mathbb R^4\), where
\[
Y=P+N,\qquad S=u(P)^\top V+\epsilon,\qquad u(p)=p/\|p\|,
\]
with \(N\sim\mathcal N(0,\sigma_N^2I_3)\), \(\epsilon\sim\mathcal N(0,\sigma_\epsilon^2)\), and all noises independent of each other and of \((P,V)\). Here \(Y\) is a noisy Cartesian position measurement and \(S\) is a noisy radial-velocity Doppler measurement. Let
\[
\mathcal F_{\mathrm{aff}}
=
\{f:\mathbb R^4\to\mathbb R^3\mid f(y,s)=a+Ay+bs,\;
a,b\in\mathbb R^3,\; A\in\mathbb R^{3\times3}\}
\]
be the class of affine position estimators from \(Z=(Y,S)\).
This representation is without loss of generality, since any affine map 
\(f(z)=a+Mz\), \(M\in\mathbb R^{3\times4}\), can be decomposed as \(Ay+bs\). Then the Bayes estimator
\[
g^\star(Z)=\mathbb E[P\mid Z]
\]
is not affine. Consequently,
\[
\inf_{f\in\mathcal F_{\mathrm{aff}}}\mathbb E\|P-f(Z)\|^2
>
\mathbb E\|P-g^\star(Z)\|^2.
\]

\paragraph{Proof.}
The conditional mean \(g^\star(Y,S)=\mathbb E[P\mid Y,S]\) is the unique minimizer of
\[
R(f)=\mathbb E\|P-f(Y,S)\|^2
\]
over all square-integrable estimators \(f(Y,S)\). We show that this minimizer cannot be affine.

Since \(P,V\) are jointly Gaussian and \(\operatorname{Cov}(V,P)=\rho I_3\), the Gaussian conditioning formula gives
\[
\mathbb E[V\mid P]
=
\operatorname{Cov}(V,P)\operatorname{Cov}(P)^{-1}P
=
\frac{\rho}{\sigma_P^2}P .
\]
Hence we may write
\[
V=\gamma P+U,\qquad \gamma:=\frac{\rho}{\sigma_P^2},
\]
where \(U:=V-\gamma P\) is Gaussian, isotropic, and independent of \(P\). Since \(\rho\neq0\), we have \(\gamma\neq0\). Therefore
\[
S
=
u(P)^\top V+\epsilon
=
u(P)^\top(\gamma P+U)+\epsilon
=
\gamma\|P\|+\eta,
\qquad
\eta:=u(P)^\top U+\epsilon .
\]
Because \(U\) is isotropic Gaussian and independent of \(P\), for every fixed \(P=p\),
\[
u(p)^\top U\sim\mathcal N(0,\sigma_U^2).
\]
Thus
\[
\eta\mid P=p\sim\mathcal N(0,\tau^2),
\qquad
\tau^2=\sigma_U^2+\sigma_\epsilon^2,
\]
and hence
\[
S\mid P=p\sim\mathcal N(\gamma\|p\|,\tau^2).
\]
Thus the Doppler measurement provides information about the range \(\|P\|\).

By Bayes' rule, for fixed \(Y=y\) and \(S=s\), the posterior density of \(P\) is proportional to
\[
p(p\mid y,s)
\propto
\exp\left(
-\frac{\|p\|^2}{2\sigma_P^2}
-\frac{\|y-p\|^2}{2\sigma_N^2}
-\frac{(s-\gamma\|p\|)^2}{2\tau^2}
\right).
\]
The last term is nonlinear in \(p\) through \(\|p\|\). We now show that this nonlinearity forces the posterior mean to depend nonlinearly on the observation.

First, the model is rotation equivariant. For every rotation matrix \(Q\in SO(3)\),
\[
(QP,QY,S)\stackrel{d}{=}(P,Y,S).
\]
Therefore, by uniqueness of the conditional mean,
\[
g^\star(Qy,s)=Qg^\star(y,s)
\]
for almost every \((y,s)\).

Now suppose, for contradiction, that \(g^\star\) is affine:
\[
g^\star(y,s)=a+Ay+bs.
\]
The equivariance identity implies
\[
a+AQy+bs
=
Q(a+Ay+bs)
\]
for every \(Q\in SO(3)\), \(y\in\mathbb R^3\), and \(s\in\mathbb R\). Comparing the constant, \(y\)-dependent, and \(s\)-dependent terms gives
\[
a=0,\qquad b=0,\qquad A=\alpha I_3
\]
for some scalar \(\alpha\). Hence any affine estimator satisfying the rotation-equivariance property must have the form
\[
g^\star(y,s)=\alpha y.
\]
In particular, such an estimator cannot depend on the Doppler observation \(s\).

We now show that the true conditional mean does depend on \(s\). Fix \(y\neq0\), and write
\[
e=\frac{y}{\|y\|}.
\]
By rotational symmetry around the axis \(e\), the posterior mean must be aligned with \(e\), so there exists a scalar \(m(\|y\|,s)\) such that
\[
g^\star(y,s)=m(\|y\|,s)e.
\]
Equivalently, for fixed \(y\neq0\), define
\[
m_y(s):=\mathbb E[e^\top P\mid Y=y,S=s].
\]
The posterior density can be written as
\[
p_s(p)
=
\frac{1}{Z(s)}
\exp\left(
-A_y(p)+\frac{\gamma s}{\tau^2}\|p\|
\right),
\]
where
\[
A_y(p)
=
\frac{\|p\|^2}{2\sigma_P^2}
+
\frac{\|y-p\|^2}{2\sigma_N^2}
+
\frac{\gamma^2\|p\|^2}{2\tau^2}.
\]
Thus \(s\) enters the posterior as an exponential tilt in the statistic \(\|P\|\). Differentiating under the integral sign gives
\[
\frac{d}{ds}m_y(s)
=
\frac{\gamma}{\tau^2}
\operatorname{Cov}_s(e^\top P,\|P\|),
\]
where the covariance is taken with respect to \(p_s(p)=p(p\mid y,s)\). This differentiation is justified by dominated convergence, since the posterior density has Gaussian tails uniformly over compact intervals of \(s\).

It remains to show that this covariance is strictly positive. Write
\[
p=r\omega,\qquad r=\|p\|\ge0,\qquad \omega\in\mathbb S^2.
\]
Since \(y=\|y\|e\), the Gaussian likelihood term satisfies
\[
\exp\left(-\frac{\|y-p\|^2}{2\sigma_N^2}\right)
=
C(r,y)
\exp\left(
\frac{r\|y\|}{\sigma_N^2}e^\top\omega
\right),
\]
where \(C(r,y)\) does not depend on \(\omega\). Hence, conditional on \(\|P\|=r\), the angular distribution is von Mises--Fisher around direction \(e\). Therefore,
\[
\mathbb E[e^\top P\mid \|P\|=r,Y=y,S=s]
=
r A\left(\frac{r\|y\|}{\sigma_N^2}\right),
\]
where
\[
A(\kappa)=\coth(\kappa)-\frac{1}{\kappa}
\]
is the mean resultant length in dimension three. Since \(A(\kappa)>0\) and is strictly increasing for \(\kappa>0\), the map
\[
r\mapsto r A\left(\frac{r\|y\|}{\sigma_N^2}\right)
\]
is strictly increasing for \(r>0\). Since the posterior distribution of \(\|P\|\) is nondegenerate, Chebyshev's association inequality gives
\[
\operatorname{Cov}_s(e^\top P,\|P\|)>0.
\]
Consequently,
\[
\frac{d}{ds}m_y(s)
=
\frac{\gamma}{\tau^2}
\operatorname{Cov}_s(e^\top P,\|P\|)
\neq0 .
\]
Hence \(m_y(s)\) is not constant in \(s\), so the Bayes estimator \(g^\star(y,s)\) depends nontrivially on the Doppler coordinate \(s\).

But rotational equivariance forced every affine equivariant estimator to have the form \(g^\star(y,s)=\alpha y\), which is independent of \(s\). This is a contradiction. Hence
\[
g^\star\notin\mathcal F_{\mathrm{aff}}.
\]

Finally, by the orthogonality principle for conditional expectation, for every square-integrable estimator \(f(Z)\),
\[
\mathbb E\|P-f(Z)\|^2
=
\mathbb E\|P-g^\star(Z)\|^2
+
\mathbb E\|g^\star(Z)-f(Z)\|^2.
\]
The affine estimator class \(\mathcal F_{\mathrm{aff}}\) is finite-dimensional and closed in
\(L^2(\sigma(Z);\mathbb R^3)\). Since \(g^\star\notin\mathcal F_{\mathrm{aff}}\), its distance from this class is strictly positive. Therefore,
\[
\inf_{f\in\mathcal F_{\mathrm{aff}}}
\mathbb E\|P-f(Z)\|^2
>
\mathbb E\|P-g^\star(Z)\|^2.
\]
This proves the claim.

\subsection{Proof of Corollary 2}

\begin{proposition}
\label{prop:LiDAR-gap}
Let
\[
X \sim \mathcal N(0,\sigma_x^2 I_2),
\]
and suppose
\[
r_z=\|X\|+V_r,
\qquad
\theta_z=\arg(X)+V_\theta,
\]
where $V_r\sim\mathcal N(0,\sigma_r^2)$ and $V_\theta\sim\mathcal N(0,\sigma_\theta^2)$ are independent, with $\sigma_\theta>0$. Assume $\theta_z$ is the real value of the angle. The reported Cartesian measurement is
\[
Z=(r_z\cos\theta_z,\; r_z\sin\theta_z)^\top.
\]
Then, for every $z\neq 0$:
\begin{enumerate}
\item $p(x\mid Z=z)$ is not Gaussian in $x$;
\item the Bayes estimator $g^*(z):=\mathbb E[X\mid Z=z]$ is not affine in $z$;
\item consequently,
\[
\inf_{f(z)=a+Bz}\mathbb E\|X-f(Z)\|^2
\;>\;
\mathbb E\|X-g^*(Z)\|^2.
\]
\end{enumerate}
\end{proposition}

\begin{proof}
Write
\[
Z=Su_B,\qquad S=r_z,\quad B=\theta_z,
\]
and let \(z=Ru_\beta\), \(R>0\). For \(x=\rho u_\phi\), the map
\((s,b)\mapsto su_b\) has Jacobian \(|s|\), and \(z=Ru_\beta\) is obtained
from the two branches
\[
(S,B)=(R,\beta+2\pi k),\qquad
(S,B)=(-R,\beta+\pi+2\pi k).
\]
Hence, with
\[
K(a)=\sum_{k\in\mathbb Z}
\exp\!\left(-\frac{(a+2\pi k)^2}{2\sigma_\theta^2}\right),
\]
Bayes' rule gives, up to a factor independent of \(x\),
\[
p(x\mid Z=z)
\propto
e^{-\rho^2/(2\sigma_x^2)}
\left[
e^{-(R-\rho)^2/(2\sigma_r^2)}K(\beta-\phi)
+
e^{-(R+\rho)^2/(2\sigma_r^2)}K(\beta+\pi-\phi)
\right].
\]

Restricting to the ray \(x=t u_\beta\), \(t>0\), gives
\[
\log p(tu_\beta\mid Z=z)
=
C-\frac{t^2}{2\sigma_x^2}
-\frac{R^2+t^2}{2\sigma_r^2}
+
\log\!\left[
K(0)e^{Rt/\sigma_r^2}
+
K(\pi)e^{-Rt/\sigma_r^2}
\right].
\]
Since \(R>0\) and \(K(0),K(\pi)>0\), the last term is not quadratic in
\(t\). Thus the restriction of the log posterior to a line is not quadratic,
so \(p(x\mid Z=z)\) is not Gaussian.

By rotational equivariance,
\[
g^*(Ru_\beta)=q(R)u_\beta .
\]
Let
\[
\kappa_\theta=\mathbb E[\cos V_\theta]
=e^{-\sigma_\theta^2/2}.
\]
Since \(g^*(Ru_\beta)\) is aligned with \(u_\beta\), we write
\[
g^*(Ru_\beta)=q(R)u_\beta .
\]
Thus
\[
q(R)
=
u_\beta^\top \mathbb{E}[X\mid Z=Ru_\beta]
=
\mathbb{E}[\rho\cos(\phi-\beta)\mid Z=Ru_\beta].
\]
Using \(x=\rho u_\phi\) and \(dx=\rho\,d\rho\,d\phi\), this gives
\[
q(R)
=
\frac{
\int_0^\infty\int_0^{2\pi}
\rho^2\cos(\phi-\beta)\,
p(\rho,\phi\mid Z=Ru_\beta)\,d\phi\,d\rho
}{
\int_0^\infty\int_0^{2\pi}
\rho\,
p(\rho,\phi\mid Z=Ru_\beta)\,d\phi\,d\rho
}.
\]
Substituting the posterior density above and integrating over \(\phi\) yields
\[
q(R)
=
\kappa_\theta
\frac{
\int_0^\infty
\rho^2 e^{-\rho^2/(2\sigma_x^2)}
\left[
e^{-(R-\rho)^2/(2\sigma_r^2)}
-
e^{-(R+\rho)^2/(2\sigma_r^2)}
\right]\,d\rho
}{
\int_0^\infty
\rho e^{-\rho^2/(2\sigma_x^2)}
\left[
e^{-(R-\rho)^2/(2\sigma_r^2)}
+
e^{-(R+\rho)^2/(2\sigma_r^2)}
\right]\,d\rho
}.
\]
Set
\[
c=\frac1{2\sigma_x^2}+\frac1{2\sigma_r^2}.
\]
Using
\[
e^{-(R-\rho)^2/(2\sigma_r^2)}
-
e^{-(R+\rho)^2/(2\sigma_r^2)}
=
2e^{-R^2/(2\sigma_r^2)}e^{-\rho^2/(2\sigma_r^2)}
\sinh\!\left(\frac{R\rho}{\sigma_r^2}\right),
\]
dominated convergence yields
\[
\lim_{R\downarrow0}\frac{q(R)}{R}
=
\kappa_\theta
\frac{
\frac{2}{\sigma_r^2}\int_0^\infty \rho^3e^{-c\rho^2}\,d\rho
}{
2\int_0^\infty \rho e^{-c\rho^2}\,d\rho
}
=
\frac{2\kappa_\theta\sigma_x^2}{\sigma_x^2+\sigma_r^2}.
\]
On the other hand, as \(R\to\infty\), the negative-range term is
exponentially negligible. Completing the square gives concentration near
\[
\rho=\alpha R,\qquad
\alpha=\frac{\sigma_x^2}{\sigma_x^2+\sigma_r^2},
\]
so
\[
\lim_{R\to\infty}\frac{q(R)}{R}
=
\kappa_\theta \alpha
=
\frac{\kappa_\theta\sigma_x^2}{\sigma_x^2+\sigma_r^2}.
\]
The two limits differ, so \(q(R)/R\) is not constant.

If \(g^*\) were affine, rotational equivariance would force
\(g^*(z)=\lambda z\). This would imply \(q(R)=\lambda R\) for all \(R>0\),
contradicting the preceding limits. Hence \(g^*\) is not affine.

Finally, for every square-integrable estimator \(f(Z)\),
\[
\mathbb E\|X-f(Z)\|^2
=
\mathbb E\|X-g^*(Z)\|^2
+
\mathbb E\|g^*(Z)-f(Z)\|^2 .
\]
The affine estimators form a finite-dimensional closed subspace of
\(L^2(\sigma(Z);\mathbb R^2)\), and \(g^*\) does not belong to it. Therefore
its distance from this subspace is positive, giving
\[
\inf_{f(z)=a+Bz}\mathbb E\|X-f(Z)\|^2
>
\mathbb E\|X-g^*(Z)\|^2 .
\]
\end{proof}

\subsection{Proof of Lemma 1}

\begin{proof}
Consider \(x_{k+1} = F x_k + w_k\) and \(\hat{x}_{k+1|k} = \tilde{F}\hat{x}_{k|k}\). Let \(e_k := x_k - \hat{x}_{k|k}\) and \(\eta_k := (F-\tilde{F})\hat{x}_{k|k} + w_k\), so that \(x_{k+1}-\hat{x}_{k+1|k} = F e_k + \eta_k\). Assume that \(e_{k+1} = F e_k + \eta_k\), \(\|F^j\|\le C\lambda^j\) for some \(C>0\), \(0<\lambda<1\), and \(\sup_k \mathbb{E}\|\eta_k\|^2 \le M_\eta\). Then, for all \(k\ge 0\),

Then, from the system and predictor definitions,

\[
x_{k+1}-\hat{x}_{k+1|k}
= F x_k + w_k - \tilde F \hat{x}_{k|k}
= F e_k + \eta_k,
\]
hence
\[
\mathbb{E}\|x_{k+1}-\hat{x}_{k+1|k}\|^2
\le 2\|F\|^2 \mathbb{E}\|e_k\|^2 + 2\mathbb{E}\|\eta_k\|^2.
\]
By assumption, \(\mathbb{E}\|\eta_k\|^2 \le M_\eta\), so it suffices to bound \(\mathbb{E}\|e_k\|^2\).

By iterating the recursion \(e_k = F e_{k-1} + \eta_{k-1}\), we obtain
\[
\begin{aligned}
e_k 
&= F\big(F e_{k-2} + \eta_{k-2}\big) + \eta_{k-1} \\
&= F^2 e_{k-2} + F \eta_{k-2} + \eta_{k-1} \\
&\;\;\vdots \\
&= F^k e_0 + \sum_{t=0}^{k-1} F^{k-1-t}\eta_t.
\end{aligned}
\]

Using \(\|a+b\|^2 \le 2\|a\|^2+2\|b\|^2\),
\[
\mathbb{E}\|e_k\|^2
\le 2\mathbb{E}\|F^k e_0\|^2 + 2\mathbb{E}\Big\|\sum_{t=0}^{k-1} F^{k-1-t}\eta_t\Big\|^2.
\]
For the first term, \(\|F^k\|\le C\lambda^k\) implies
\[
\mathbb{E}\|F^k e_0\|^2 \le C^2 \lambda^{2k}\,\mathbb{E}\|e_0\|^2.
\]
For the second term, by the triangle inequality and Cauchy–Schwarz,
\[
\mathbb{E}\Big\|\sum_{t=0}^{k-1} F^{k-1-t}\eta_t\Big\|^2
\le
\Big(\sum_{j=0}^{k-1} \|F^j\| \sqrt{\mathbb{E}\|\eta_{k-1-j}\|^2}\Big)^2
\le
M_\eta \Big(\sum_{j=0}^{k-1} C\lambda^j\Big)^2.
\]
Using \(\sum_{j=0}^{k-1}\lambda^j \le (1-\lambda)^{-1}\), we obtain
\[
\mathbb{E}\|e_k\|^2
\le
2 C^2 \lambda^{2k}\,\mathbb{E}\|e_0\|^2
+
\frac{2C^2}{(1-\lambda)^2} M_\eta.
\]
Substituting back,
\[
\mathbb{E}\|x_{k+1}-\hat{x}_{k+1|k}\|^2
\le
4\|F\|^2 C^2 \lambda^{2k}\,\mathbb{E}\|e_0\|^2
+
\left(\frac{4\|F\|^2 C^2}{(1-\lambda)^2}+2\right) M_\eta.
\]
Taking \(\limsup_{k\to\infty}\) gives:
\[
\limsup_{k\to\infty}\mathbb{E}\|x_{k+1}-\hat{x}_{k+1|k}\|^2
\le
\left(\frac{4\|F\|^2 C^2}{(1-\lambda)^2}+2\right)M_\eta.
\]

This proves the result.
\end{proof}
\newpage

\section{Prompt}
\label{Appendix:Prompt}
\begin{figure}[h!]
\centering
\noindent
\begin{minipage}[t]{0.99\textwidth}
  \begin{pythonbox}[height=19.cm]
  \begin{lstlisting}[style=pythonstyle]
### Task
Generate ten diverse Python implementations of a function named `approximate`
that performs a Kalman-like PREDICT -> UPDATE step for LiDAR tracking.

### Required signature
    def approximate(F, H, Q, R, x_upd_prev, P, z):
### Required return
    (x_upd_new, P_upd_new, x_pred, P_pred)

Where:
- x_pred     = predicted state BEFORE update (x_{t|t-1})
- P_pred     = predicted covariance BEFORE update
- x_upd_new  = updated state AFTER incorporating measurement
- P_upd_new  = updated covariance after measurement

### Constraints
- Use NumPy only (no external Kalman libraries).
- Implement a deterministic function (no random components for the same inputs).
- Follow the PREDICT -> UPDATE order.

#Example 1
def approximate(F, H, Q, R, x_upd_prev, P, z):
    x_pred = F @ x_upd_prev
    P_pred = F @ (P @ F.T) + Q
    y = z - H @ x_pred
    magnitude = np.linalg.norm(y)
    scaling_factor = 0.01 * magnitude
    R_scaled = R * scaling_factor
    S_innov = H @ (P_pred @ H.T) + R_scaled
    inv_S = np.linalg.inv(S_innov)
    K = P_pred @ H.T @ inv_S
    x_upd_new = x_pred + K @ y
    I = np.eye(F.shape[0])
    P_upd_new = (I - K @ H) @ P_pred
    return (x_upd_new, P_upd_new, x_pred, P_pred) 
    
#Example 2
def approximate(F, H, Q, R, x_upd_prev, P, z):
    # Prediction step
    x_pred = F @ x_upd_prev
    P_pred = F @ (P @ F.T) + Q
    # Innovation
    y = z - H @ x_pred
    # Innovation covariance
    S = H @ (P_pred @ H.T) + R
    # Kalman gain using np.linalg.solve for efficiency
    K = np.linalg.solve(S, H @ P_pred.T).T
    # Update step
    x_upd_new = x_pred + K @ y
    # Update covariance
    P_upd_new = P_pred - K @ H @ P_pred
    P_upd_new += K @ R @ K.T
    return (x_upd_new, P_upd_new, x_pred, P_pred)
  \end{lstlisting}
  \end{pythonbox}
\end{minipage}%
\caption{{ Prompt provided to the language model in the LLM-assisted evolutionary search framework.}}
\label{fig:prompt}

\end{figure}

\newpage
\section{Additional experiments }
\label{Appendix:Additional_experiments}
This appendix details experiments on synthetic Doppler radar tracking, LiDAR-based state estimation, and real-world pedestrian trajectory prediction, covering both state estimation and next state prediction tasks. These include case studies on specific tasks, out of distribution generalization generalization studies, and statistical tests (e.g., z-tests) to confirm that observed improvements are significant and robust.

\subsection{LiDAR}

\label{Appendix:LiDAR}
\begin{figure}[h]
    \centering

    \begin{subfigure}[t]{0.24\textwidth}
        \centering
        \includegraphics[height=2.5cm,width=\linewidth]{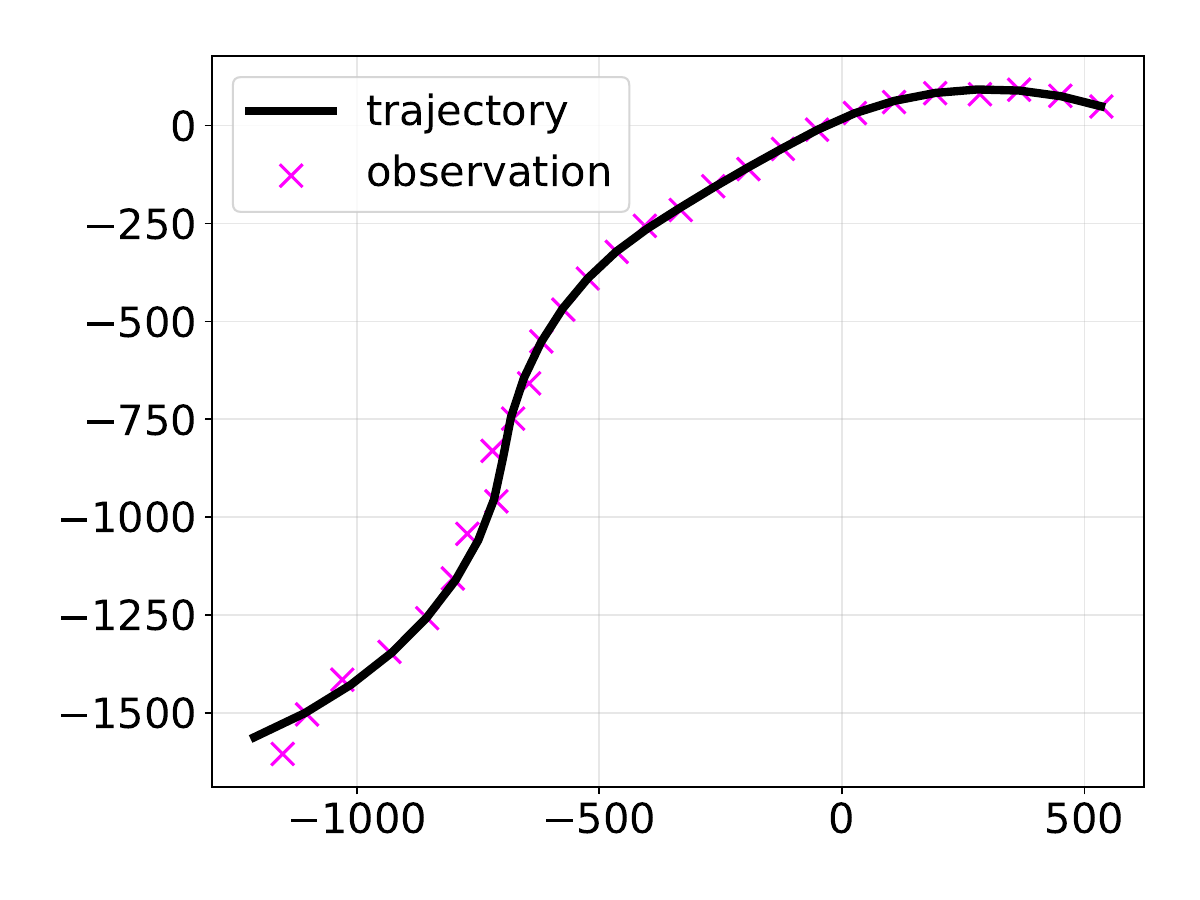}
    \end{subfigure}
    \hfill
    \begin{subfigure}[t]{0.24\textwidth}
        \centering
        \includegraphics[height=2.5cm,width=\linewidth]{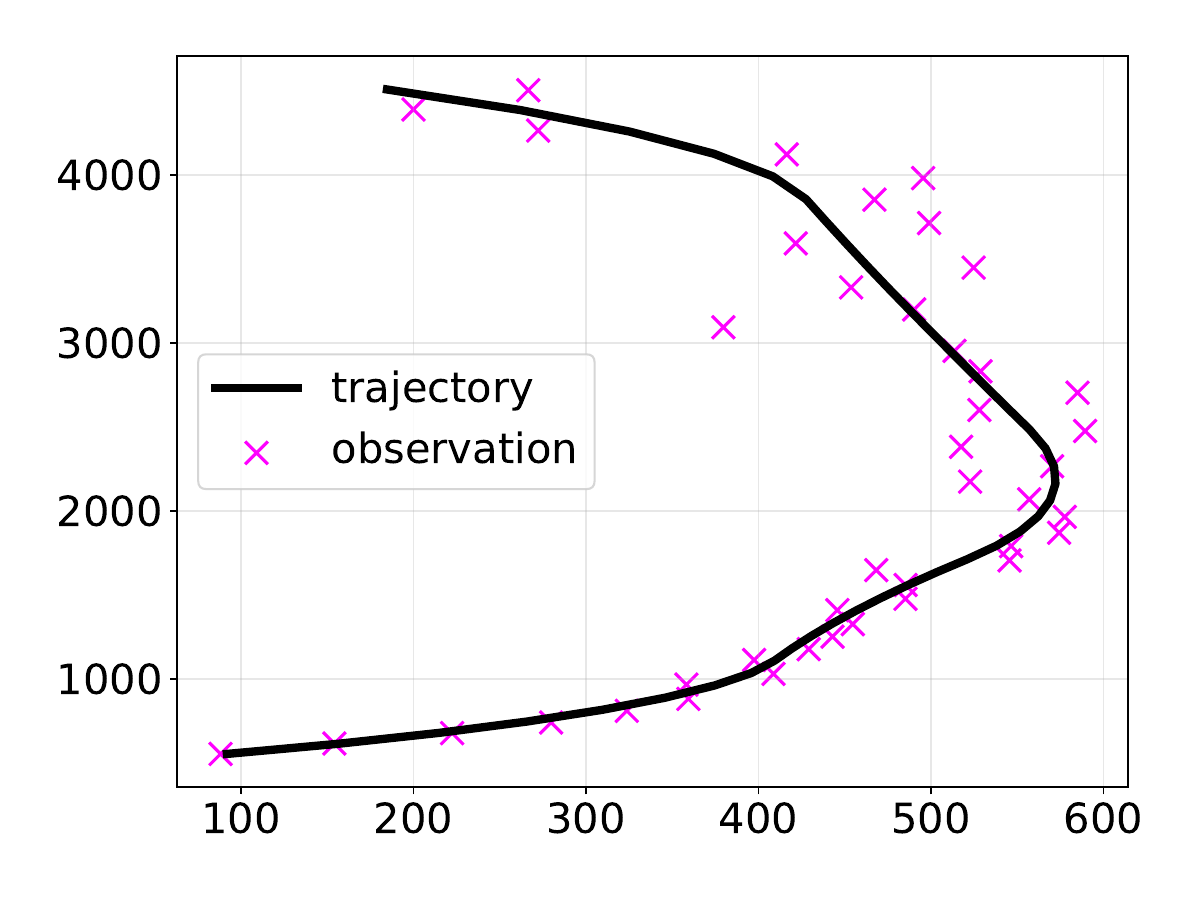}
    \end{subfigure}
    \hfill
    \begin{subfigure}[t]{0.24\textwidth}
        \centering
        \includegraphics[height=2.5cm,width=\linewidth]{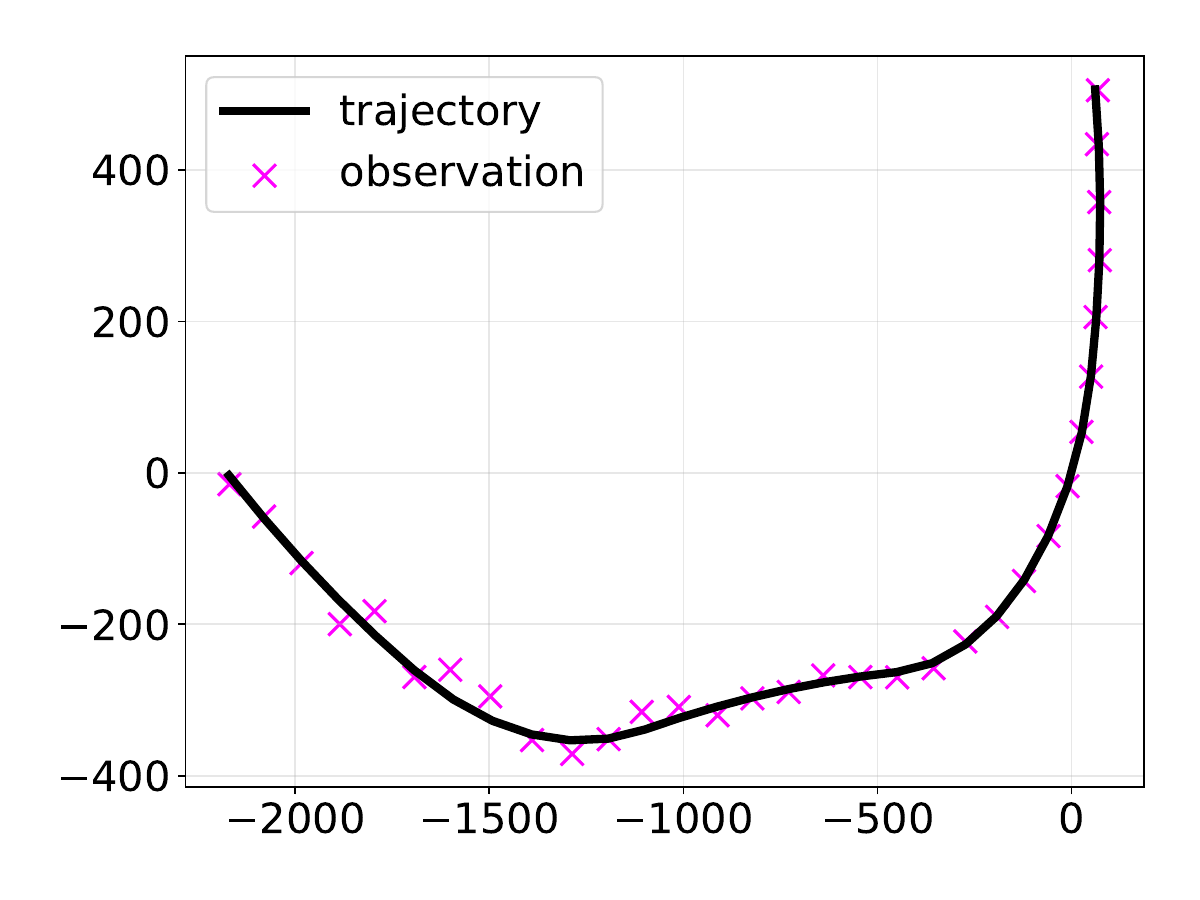}
    \end{subfigure}
    \hfill
    \begin{subfigure}[t]{0.24\textwidth}
        \centering
        \includegraphics[height=2.5cm,width=\linewidth]{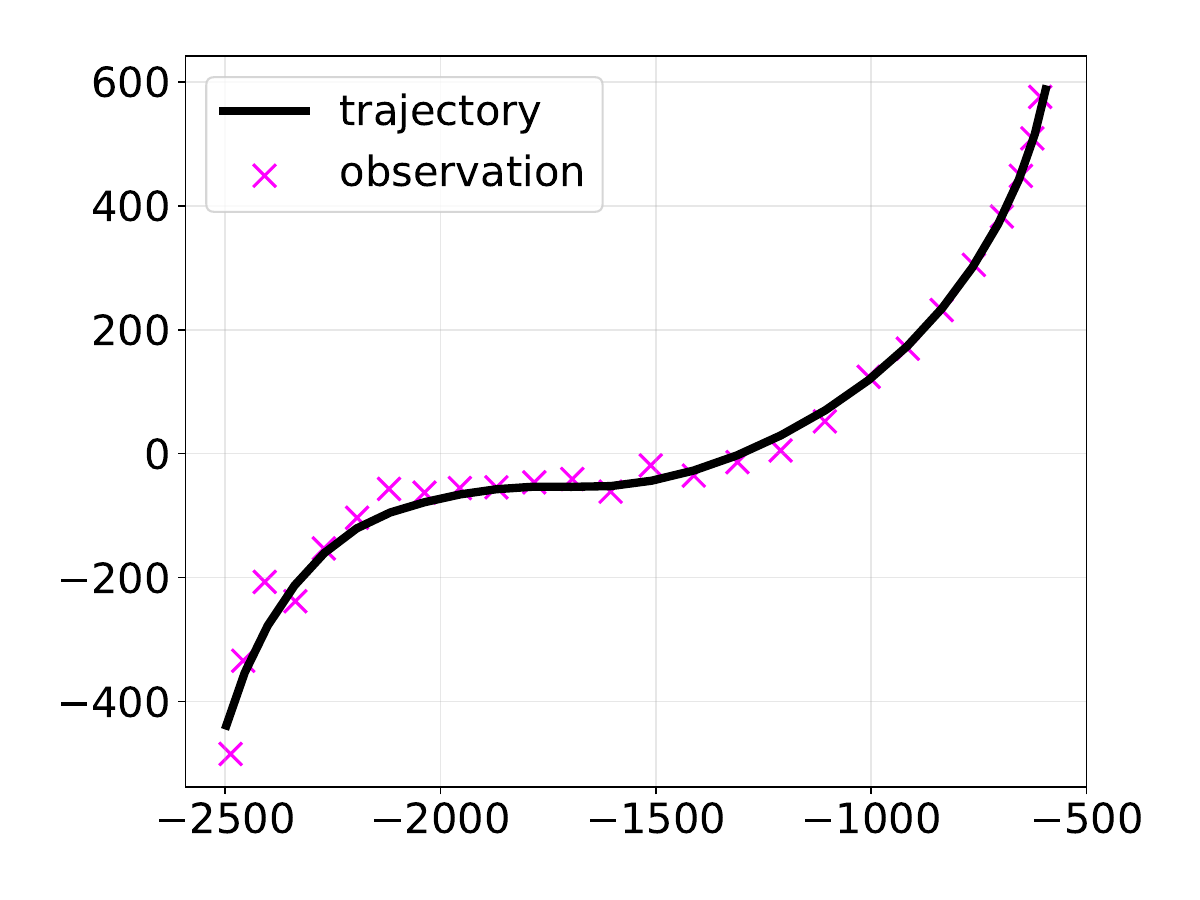}
    \end{subfigure}
    \hfill

    \caption{LiDAR trajectory examples}
\label{fig:LiDAR_trajectoris}
\end{figure}

We simulate $N=2000$ independent 2D trajectories with latent states $x_t = (p_t^x, p_t^y, v_t^x, v_t^y)$. Initial positions are sampled uniformly in a bounded region, and velocities are initialized with 
random direction and magnitude. In figure \ref{fig:LiDAR_trajectoris} we visualize some of the trajectories. Dynamics evolve over piecewise-constant intervals. For each interval, 
we sample its length and radial/tangential accelerations $(a_r, a_t)$. 
Let $v_t = \|(v_t^x, v_t^y)\|$. The acceleration is:
\[
a_t = \left(a_r \frac{v_t^x}{v_t} - a_t \frac{v_t^y}{v_t},\;
            a_r \frac{v_t^y}{v_t} + a_t \frac{v_t^x}{v_t}\right),
\]
and the state evolves as
\[
p_{t+1} = p_t + v_t + \tfrac{1}{2} a_t, \quad
v_{t+1} = v_t + a_t.
\]

Observations are generated via noisy range--bearing measurements:
\[
r_t = \|p_t\| + \epsilon_r, \quad
\theta_t = \mathrm{atan2}(p_t^y, p_t^x) + \epsilon_\theta,
\]
with Gaussian noise, and mapped to Cartesian coordinates:
\[
z_t = (r_t \cos\theta_t,\; r_t \sin\theta_t).
\]

The dataset is split into 1200 training, 300 validation, and 500 test trajectories.

\begin{table}[h]
\centering
\setlength{\tabcolsep}{4.3pt}   
\begin{tabular}{|c|c|c|c|c|c|c|c|c}
\hline
 & \begin{tabular}{c}KF\end{tabular}
 & \begin{tabular}{c}OKF\end{tabular}
 & \begin{tabular}{c}KE\end{tabular}
 & \begin{tabular}{c}OKE\end{tabular}
 & LSTM
 & \begin{tabular}{c}KN\end{tabular}
 & OBS \\
\hline
SE 
& 12.70 $\pm$ 0.1
& 11.16 $\pm$ 0.2
& {\bf 10.52}  $\pm$ 0.2
& 10.77 $\pm$ 0.2
& 20.22 $\pm$ 0.1
& 14.81 $\pm$ 0.1
& 18.15 $\pm$ 0.3\\
\hline

NSP 
& 27.80 $\pm$ 0.2
& 23.94 $\pm$ 0.2
& \textbf{22.40} $\pm$ 0.2
& 22.72  $\pm$ 0.3
& 31.52  $\pm$ 0.2
& 29.62  $\pm$ 0.2
& 76.53 $\pm$ 0.4\\

\hline
\end{tabular}

\caption{Root mean squared error performance of different methods $\pm$ stderr }
\label{tab:LiDAR}
\end{table}

Table \ref{tab:LiDAR} shows that optimizing the noise covariances $Q$ and $R$ significantly impacts Kalman filtering performance. While the OKF improves over the standard KF, the inisializing the  LLM assisted evolutionary search with Q and R obtained by least squares KE achieves the best performance by 5\% on SE and  5\% on NSP metrics, outperforming OKF despite the latter using learned noise parameters. This suggests that improved initialization or joint optimization of $Q$ and $R$ remains critical, and that combining multiple initialization strategies (e.g., least-squares and OKF) is important to reduce the performance gap.

Among learning-based baselines, KalmanNet (KN) improves over classical LSTM but remains inferior to KF, OKF, KE and OKF, indicating limited gains from end-to-end learning in this setting. LSTM performs the worst, likely due to difficulties in capturing the underlying physical structure and noise characteristics of the system. Overall, methods that explicitly leverage system structure (KE, OKF) demonstrate superior performance compared to purely data-driven approaches.
\begin{table}[h]
\centering
\setlength{\tabcolsep}{4.3pt}   
\begin{tabular}{|c|c|c|c|c|c|c|c|c}
\hline
 & \begin{tabular}{c}KF\end{tabular}
 & \begin{tabular}{c}OKF\end{tabular}
 & \begin{tabular}{c}KE\end{tabular}
 & \begin{tabular}{c}OKE\end{tabular}
 & LSTM
 & \begin{tabular}{c}KN\end{tabular}
 & OBS \\
\hline

SE  
& \textbf{2.87} 
& \textbf{2.87 } 
& 3.1  
& 3 
& 17.43 
& 49.51  
& 1.65  \\
\hline

NSP  
& \textbf{2.59 }
& \textbf{2.58 }
& 3.90 
& 2.92 
& 12.07 
& 212.0  
& 1.55 \\
\hline
\end{tabular}

\caption{Test time in seconds of different methods}
\label{tab:LiDAR_time}
\end{table}

Table \ref{tab:LiDAR_time} reports runtime in seconds. Classical filtering methods (KF, OKF, KE, OKE) exhibit comparable and very low computational cost ($\sim$2.5--3s), with KE introducing only a minor overhead relative to KF/OKF. In contrast, learning-based approaches are significantly slower: LSTM incurs an order-of-magnitude increase, while KalmanNet (KN) is substantially more expensive.Although evolutionary or LLM aided evolutionary search  may introduce additional computational overhead through non linear operators, this cost is incurred offline and does not affect inference-time efficiency. Overall, model-based approaches maintain a strong advantage in runtime while achieving competitive or superior accuracy.

\begin{figure}[h]
    \centering

    \begin{subfigure}[t]{0.45\textwidth}
        \centering
        \includegraphics[width=\linewidth]{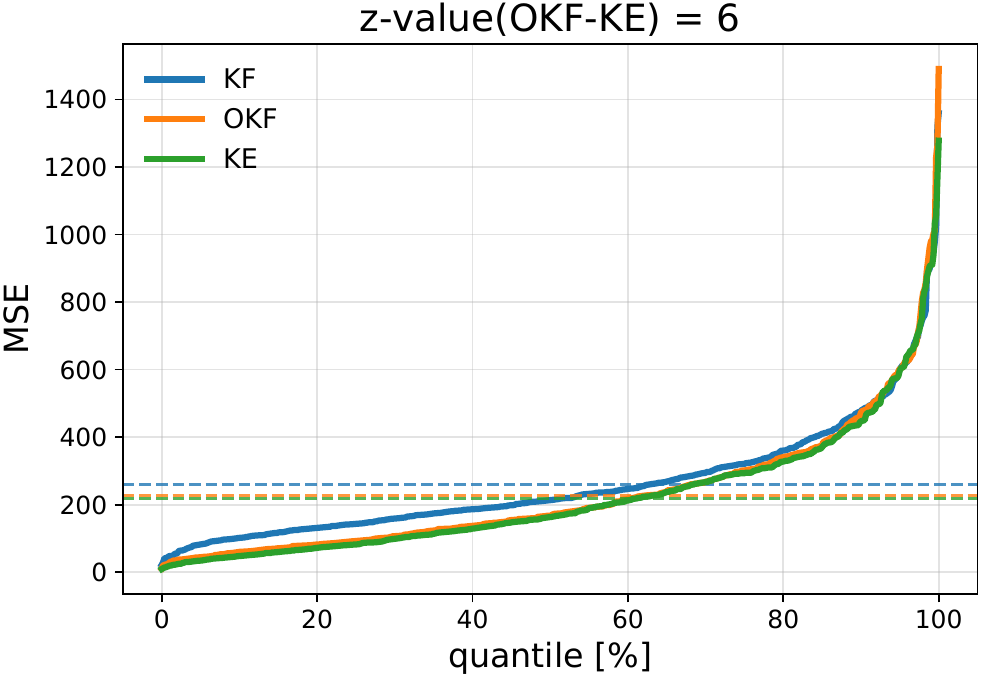}
    \end{subfigure}
    \hfill
    \begin{subfigure}[t]{0.45\textwidth}
        \centering
        \includegraphics[width=\linewidth]{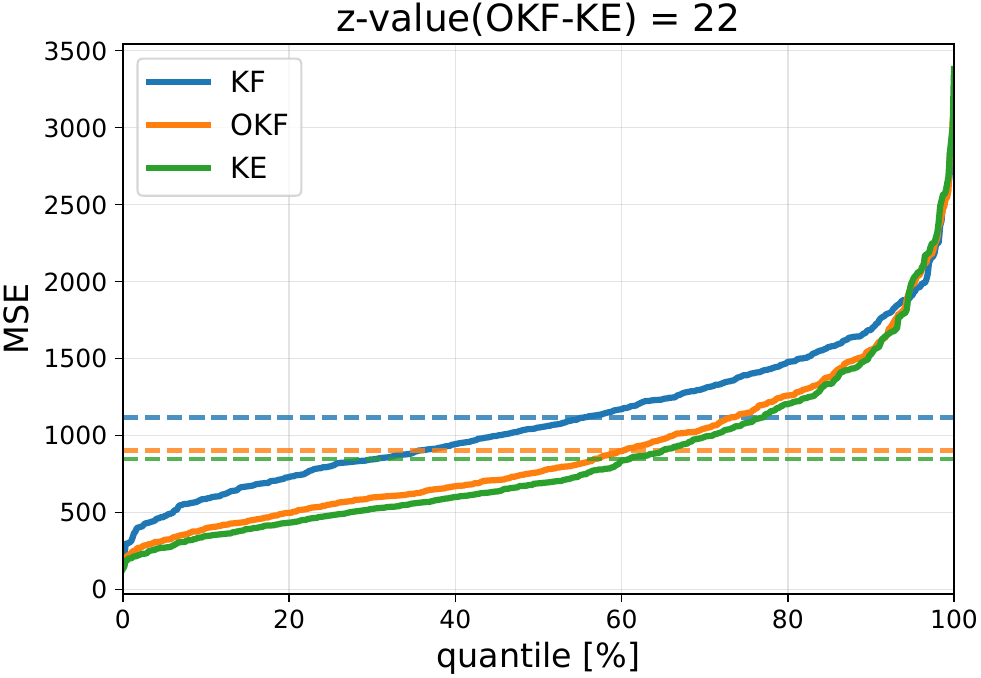}
    \end{subfigure}
    \hfill

    \caption{LiDAR Statistical Tests}
\label{fig:LiDAR_statistical_test}
\end{figure}
Figure \ref{fig:LiDAR_statistical_test} shows quantile plots of per-timestep MSE, revealing strongly heavy-tailed error distributions. The quantile curves show that KE (our method) consistently outperforms KF and OKF, particularly at higher quantiles. In the state estimation setting, gains are small despite statistical significance (z = 6). In the next state prediction, KE yields substantial improvements, especially in the upper tail (hard cases), reflected by a large z-value (z = 22).

\subsection{NCLT}
\label{Appendix:NCLT}
The NCLT dataset \cite{10.1177/0278364915614638} is a long-term robotics dataset collected on the University of Michigan campus using a Segway platform equipped with LiDAR and multiple onboard sensors. GPS measurements are obtained from a consumer-grade receiver, while ground-truth poses are provided by a high-accuracy GPS/INS system through post-processed sensor fusion. In practice, consumer-grade GPS measurements degrade significantly or become unavailable in challenging environments—particularly during indoor transitions—introducing substantial observation noise and model mismatch, as illustrated in Fig.~\ref{fig:Michigan_Campus}.

\begin{figure}[h!]
    \centering
    \includegraphics[width=\linewidth]{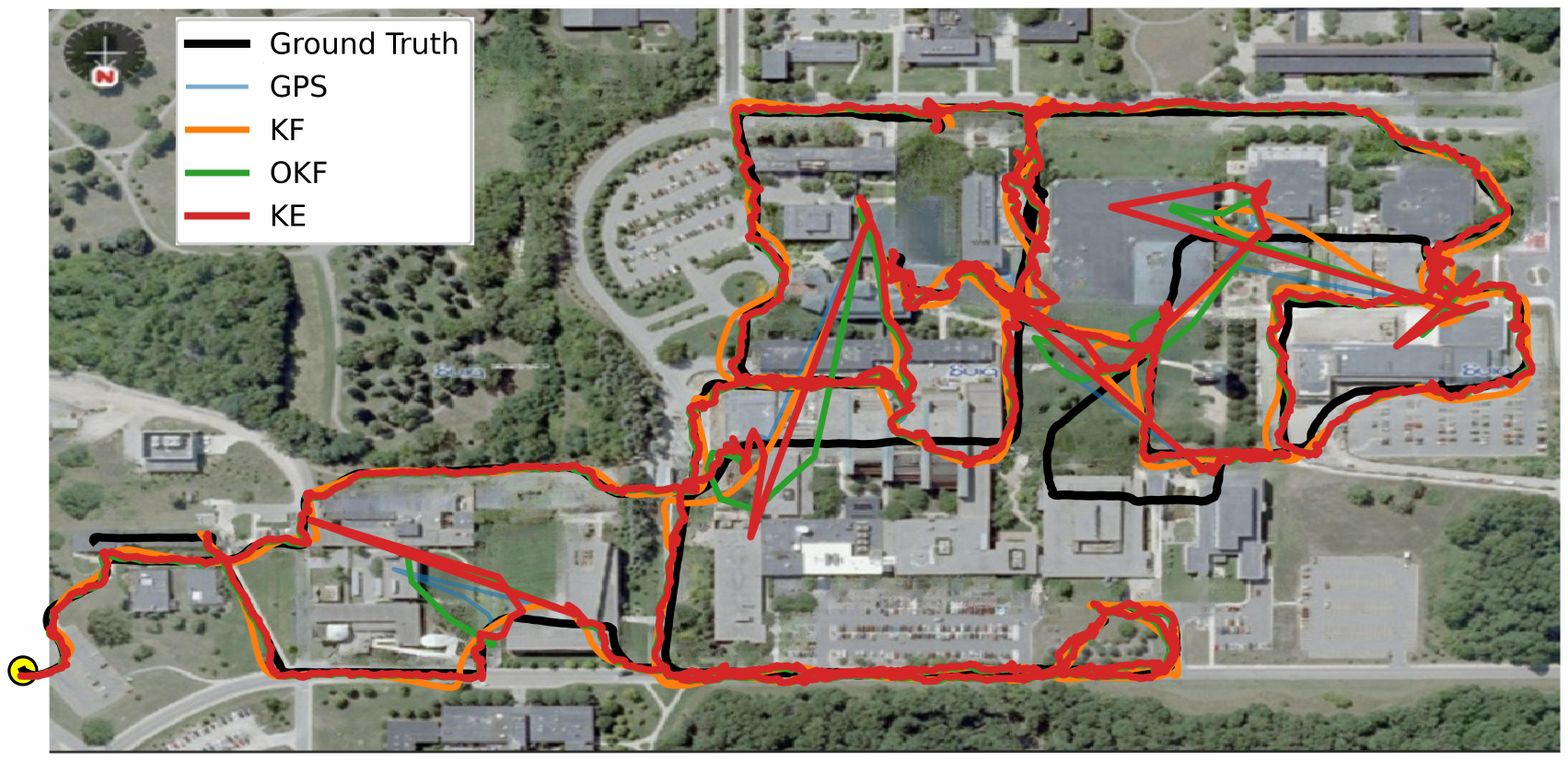}
    \caption{Michigan Campus. }
    \label{fig:Michigan_Campus}
\end{figure}

The  trajectories consist of the state is the vehicle’s 2D location and velocity, and $\tilde{F}$ is modeled according to a constant-velocity model. 
The observation (both true $H$ and modeled $\tilde{H}$) corresponds to the location. This results in the following model:
\small
\[
\tilde{F} =
\begin{pmatrix}
1 & 0 & 1 & 0 \\
0 & 1 & 0 & 1 \\
0 & 0 & 1 & 0 \\
0 & 0 & 0 & 1
\end{pmatrix}\, , \,\,\,\,\,\,\,\,
\tilde{H} = H =
\begin{pmatrix}
1 & 0 & 0 & 0 \\
0 & 1 & 0 & 0
\end{pmatrix}.
\]

The dataset consists of 27 trajectories, which we divide into 21 for training, 3 for validation, and 3 for testing. Due to this limited number of trajectories and the inherent noise in the measurements, the dataset may not be ideally suited for learning complex dynamical models Nevertheless, as shown in Table~\ref{table:NCLT_RESULTS}, the optimized Kalman Filter (OKF)  improves significantly  over the standard KF, highlighting the benefit of parameter tuning. More importantly, our proposed methods, KE and OKE, consistently outperform both KF and OKF in both state estimation and next state prediction.

Moreover, OKE achieves a performance gain of approximately 2\% over OKF in SE  and 3\% in NSP . While these improvements are moderate, they are consistent across tasks and are obtained on top of an already optimized baseline, indicating that our method provides a more accurate representation of the underlying system dynamics beyond conventional Kalman filtering approaches.
\begin{table}[h]
\centering
\setlength{\tabcolsep}{4.3pt}
\begin{tabular}{|c|c|c|c|c|c|c|c|}
\hline
 & KF & OKF & KE & OKE & LSTM & KN & OBS \\
\hline
SE
& 26.3 $\pm$ 0.2
& 23.0 $\pm$ 0.2
& 22.7 $\pm$ 0.2
& \textbf{22.6} $\pm$ 0.2
& 166.1 $\pm$ 0.5
& 49.4 $\pm$ 0.3
& 49.7 $\pm$ 0.3 \\
\hline
NSP
& 26.8 $\pm$ 0.2
& 23.1 $\pm$ 0.2
& 22.8 $\pm$ 0.2
& \textbf{22.5} $\pm$ 0.2
& 53.2 $\pm$ 0.3
& 49.8 $\pm$ 0.3
& 50.0 $\pm$ 0.3 \\
\hline
\end{tabular}

\caption{RMSE comparison of different methods for NCLT dataset for SE and NSP.}

\label{table:NCLT_RESULTS}

\end{table}

Acknowledgment : This work uses the NCLT Dataset (University of Michigan North Campus Long-Term Vision and LiDAR Dataset), made available under the Open Database License (ODbL). We additionally cite Carlevaris-Bianco et al. (2016) as requested by the dataset authors.

\subsection{Doppler}
\label{Appendix:Doppler}

In this section follow the experimental protocol of OKF paper. We consider five benchmarks of increasing complexity: \textit{Toy}, \textit{Close}, \textit{Const\_v}, \textit{Const\_a}, and \textit{Free Motion}. These benchmarks are designed to systematically introduce violations of standard Kalman filtering assumptions.
\begin{table}[h]
\centering
\small
\begin{tabular}{|l|c|c|c|c|c|}
\hline

\textbf{Benchmark} & \textbf{Aniso.} & \textbf{Polar} & \textbf{Uncent.} & \textbf{Accel.} & \textbf{Turns} \\
\hline
Toy         & $\times$ & $\times$ & $\times$ & $\times$ & $\times$ \\
Close       & \checkmark & \checkmark & $\times$ & $\times$ & $\times$ \\
Const\_v    & \checkmark & \checkmark & \checkmark & $\times$ & $\times$ \\
Const\_a    & \checkmark & \checkmark & \checkmark & \checkmark & $\times$ \\

Free Motion & \checkmark & \checkmark & \checkmark & \checkmark & \checkmark \\
\hline
\end{tabular}
\caption{Benchmark definitions. Each column indicates whether a given property is present in the scenario.}
\label{tab:benchmarks}
\end{table}
Each benchmark is defined as a subset of the following properties:
\begin{itemize}
    \item \textbf{Anisotropic}: horizontal motion is more likely than vertical motion.
    \item \textbf{Polar}: noise is generated i.i.d.\ in spherical coordinates.
    \item \textbf{Uncentered}: targets are distributed far from the origin.
    \item \textbf{Acceleration}: motion includes changes in velocity.
    \item \textbf{Turns}: trajectories are not constrained to straight lines.
\end{itemize}

The \textit{Toy} benchmark represents a simplified setting with minimal assumption violations, while the \textit{Free Motion} benchmark corresponds to a realistic tracking scenario with multiple turns and accelerations. Intermediate benchmarks progressively increase the level of complexity. The benchmark definitions are summarized in Table~\ref{tab:benchmarks}.

For each benchmark, we simulate target trajectories according to the corresponding properties and generate observations using the specified noise model. The models are trained on 1200 trajectories, validated on 300 trajectories, and evaluated on 1000 unseen trajectories.
Performance is measured using the root mean squared error (RMSE) between the estimated and true target states.

\begin{table*}[h]
\centering
\small
\begin{tabular}{|c|c|c|c|c|c|c|c|}
\hline
\textbf{Benchmark} & \textbf{KF} & \textbf{OKF} & \textbf{KE} & \textbf{OKE} & \textbf{KN} & \textbf{LSTM} & \textbf{OBS} \\
\hline
Toy      & 109.24 & 78.30 & 82.69  & \textbf{78.03}  & 99.63  & 110.23 & 173.10 \\
\hline
Close    & 19.71  & 19.73 & \textbf{18.62}  &  \textbf{18.62 } & 26.41  & 35.81  & 44.19  \\
\hline
Const\_v & 85.81  & 83.50 &  \textbf{73.95}  & 76.75  & 111.69 & 119.69 & 198.41 \\
\hline
Const\_a & 95.57  & 91.78 & 82.67  &  \textbf{ 81.63}  & 123.65 & 136.27 & 216.58 \\
\hline
Free     & 101.90 & 95.72 & 86.71  &  \textbf{84.23}  & 151.20 & 111.45 & 186.42 \\
\hline
\end{tabular}
\caption{RMSE comparison of different methods for Doppler radar tracking and state estimation across benchmarks. Lower values indicate better performance.}
\label{tab:DopplerSE}
\end{table*}
Table~\ref{tab:DopplerSE} reports the RMSE across all benchmarks for state estimation. \textbf{KE} and \textbf{OKE} consistently achieve the best performance, outperforming prior methods across all scenarios.

In the Toy benchmark, OKE achieves the lowest RMSE, improving over OKF by 0.34\%. 
In Close, KE and OKE attain the best performance, reducing error by 5.5\% relative to the strongest baseline. 
For Const\_v, KE yields the largest gain, improving over the next best method by 11.4\%. 
In Const\_a, OKE achieves the best result with a 1.3\% improvement over KE. 
Finally, in Free, OKE outperforms all methods with a 2.9\% gain over the next best approach. Overall, these results demonstrate that KE, and especially OKE, provide consistent and robust gains across diverse motion settings.
KalmanNet exhibits weaker performance across benchmarks, often underperforming both KF and LSTM. This is consistent with prior observations that its reported gains depend on comparisons to non-optimized Kalman Filters. In contrast, optimized filtering approaches (OKF, KE, OKE) remain highly competitive and consistently outperform learned baselines when strong model structure is available.

Overall, these results highlight that KE and OKE provide consistent and robust gains over state-of-the-art baselines, particularly in scenarios with increasing dynamical complexity and model mismatch.

\begin{table}[h]
\centering
\begin{tabular}{|c|c|c|c|c|c|c|c|}
\hline
\textbf{Benchmark} & \textbf{KF} & \textbf{OKF} & \textbf{KE} & \textbf{OKE} & \textbf{KN} & \textbf{LSTM} & \textbf{OBS} \\
\hline
Toy     & 132.37 & 87.54  & 105.33 & \textbf{86.93} & 147.01 & 127.07 & 222.18 \\
\hline
Close   & 23.92  & 24.75  & \textbf{22.70} & 23.06 & 37.46  & 35.72  & 84.08 \\
\hline
Const\_v & 94.80  & 89.352 & \textbf{85.4}  & 85.6  & 127.35 & 143.02 & 211.15 \\
\hline
Const\_a & 109.78 & 105.13 & 99.5   & \textbf{94.65} & 144.48 & 164.24 & 244.52 \\
\hline
Free    & 125.78 & 112.21 & 102.95 & \textbf{102.76} & 133.66 & 151.47 & 205.15 \\
\hline
\end{tabular}
\caption{RMSE comparison of different methods for Doppler radar tracking and next state prediction across benchmarks. Lower values indicate better performance.}
\label{tab:DopplerNSP}
\end{table}

Table~\ref{tab:DopplerNSP} reports RMSE across all benchmarks for next state prediction. 
Consistent with the state estimation results, \textbf{KE} and \textbf{OKE} achieve the strongest overall performance, outperforming both classical and learning-based baselines across all scenarios.
In \textit{Toy}, OKE improves over OKF by 0.7\%. 
In \textit{Close}, KE achieves the best result, reducing error by 5.1\% relative to the next best method. 
For \textit{Const\_v}, KE improves over OKF by 4.4\%. 
In \textit{Const\_a}, OKE yields a 10.0\% reduction in RMSE over OKF, representing the largest gain. 
Finally, in \textit{Free}, OKE improves over OKF by 8.4\%.
Overall, gains are consistent across all tasks, remaining modest in simpler settings and becoming more pronounced as dynamical complexity increases.
\begin{figure}[h]
    \centering

    \begin{subfigure}[t]{0.48\textwidth}
        \centering
        \includegraphics[width=\linewidth]{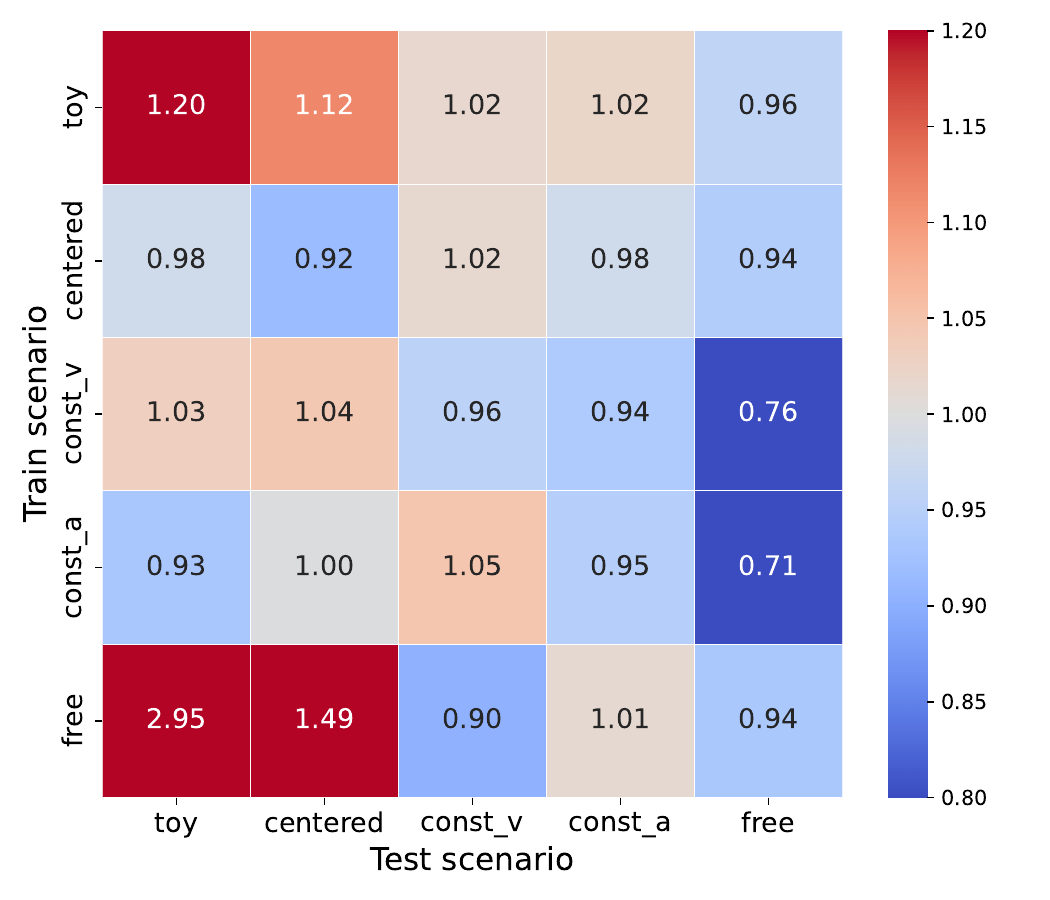}
        \caption{KE RMSE/ OKF RMSE}
        \label{fig:ke}
    \end{subfigure}
    \hfill
    \begin{subfigure}[t]{0.48\textwidth}
        \centering
        \includegraphics[width=\linewidth]{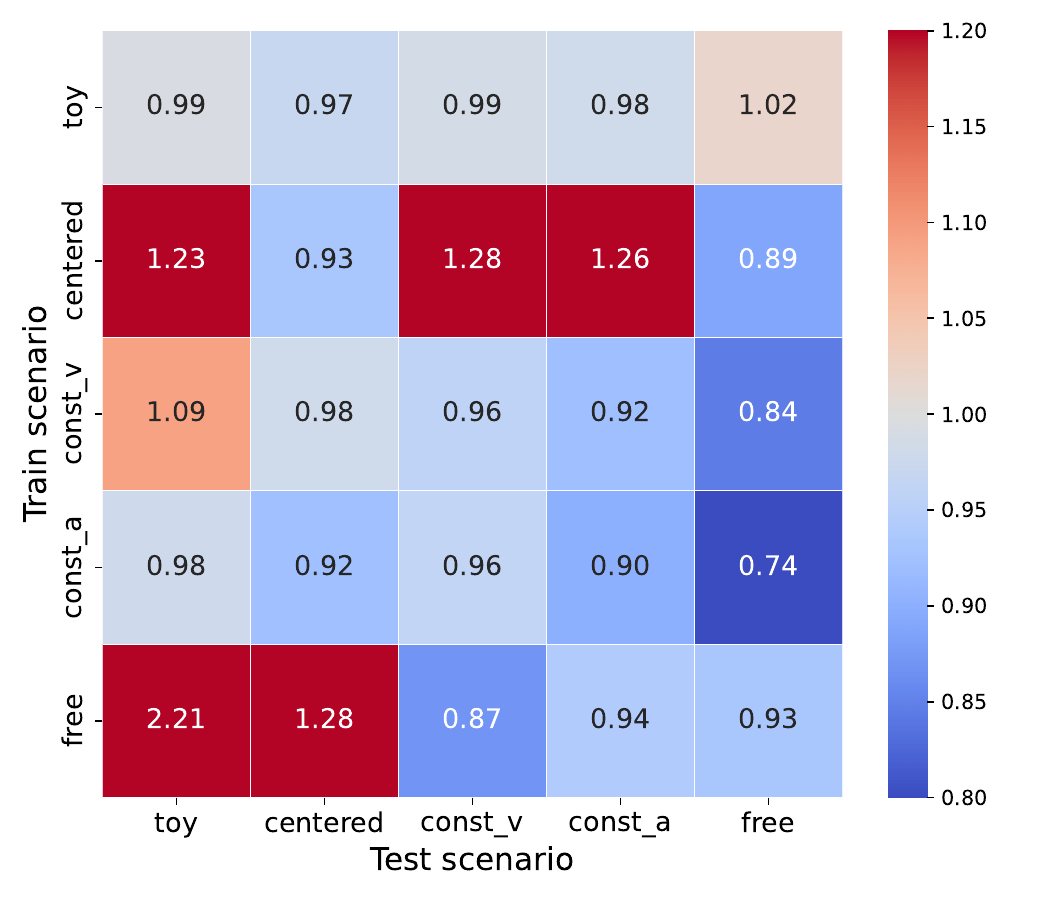}
        \caption{OKE RMSE/ OKF RMSE}
        \label{fig:oke}
    \end{subfigure}

    \caption{Generalization study in next state prediction}
    \label{fig:generaliztion_study}
\end{figure}

Figure ~\ref{fig:generaliztion_study} presents the generalization performance of KE and OKE across training and testing scenarios, reported as RMSE ratios relative to OKF (values below 1 indicate improvement). Overall, both methods demonstrate strong generalization, with most cross-scenario evaluations achieving performance comparable to or better than OKF.

KE shows stable behavior across similar dynamics, particularly between structured settings such as const\_v and const\_a, where it consistently achieves improvements (ratios < 1). However, its performance degrades when trained on the free scenario and tested on simpler regimes (e.g., toy), indicating sensitivity to mismatched dynamics.

OKE exhibits more consistent generalization across diverse settings. Notably, it maintains improvements (or near-parity) in most cross-scenario transfers, and shows stronger robustness than KE when generalizing from complex (free) to simpler scenarios. In challenging mismatched cases, such as training on centered and testing on const\_v, OKE remains competitive despite some degradation.

Overall, the results suggest that both methods generalize well across related dynamics, with OKE providing more robust performance under distribution shifts, especially when transferring between structurally different motion regimes.

\begin{figure}[h]
    \centering
    {\small
    \textcolor{KEcolor}{\rule[0.5ex]{1em}{0.2em}}\, KE \hspace{1em}
    \textcolor{yellow}{\rule[0.5ex]{1em}{0.2em}}\, OKE \hspace{1em}
    \textcolor{OKFcolor}{\rule[0.5ex]{1em}{0.2em}}\, OKF \hspace{1em}
    \textcolor{KFcolor}{\rule[0.5ex]{1em}{0.2em}}\, KF
    }

    \vspace{0.1em} 

    \includegraphics[width=0.7\linewidth]{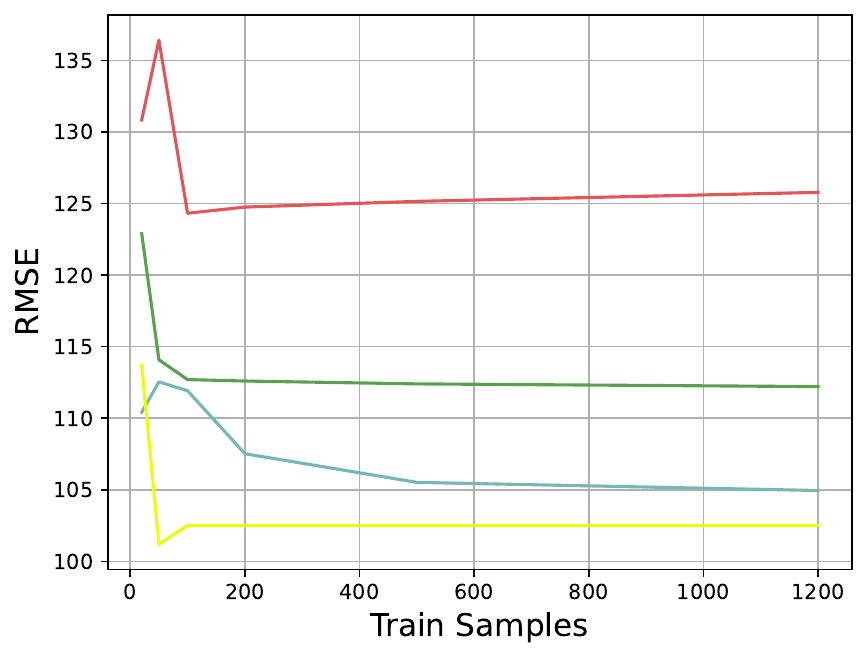}
    \caption{ RMSE as a function of training samples}
    \label{fig:performance_per_sample_free}
\end{figure}

Figure \ref{fig:performance_per_sample_free} shows RMSE as a function of the number of training samples, ranging from 20 to 1000. KE and OKE perform well even with a small number of samples, achieving low error early and improving as more data becomes available. In the low-data regime, OKE exhibits a brief instability at the smallest sample size due to high estimation variance, followed by a sharp reduction in RMSE and rapid convergence. KE follows a similar trend but improves more gradually, requiring additional data to reach comparable performance.

As the number of samples increases, both KE and OKE stabilize, with OKE consistently achieving the lowest RMSE across all regimes. In contrast, KF and OKF remain largely insensitive to the number of training samples, resulting in flat error curves and consistently higher RMSE. Overall, these results demonstrate that KE and especially OKE are data-efficient and robust, providing strong performance from as few as 20 samples.

\begin{figure}[h]
\centering
\noindent
\begin{minipage}[t]{0.49\textwidth}
  \begin{pythonbox}[height=8.cm]
  \begin{lstlisting}[style=pythonstyle]
def approximate(F, h_fun, Q, R, x_upd, P_upd, z):
    H = h_fun(x_upd, z)
    y = z - H @ x_upd
    y = np.maximum(y, -10) + np.minimum(y, 10)
    S = H @ P_upd @ H.T + R
    eps = 1e-8
    S_inv = np.linalg.inv(S + eps * np.eye(S.shape[0]))
    K = (P_upd @ H.T) @ S_inv
    x_upd = x_upd + K @ y
    x_upd = x_upd * 0.9 + (H.T @ z) * 0.1
    I = np.eye(F.shape[0])
    P_upd = (I - K @ H) @ P_upd @ (I - K @ H).T + K @ R @ K.T
    P_upd = P_upd * 0.8
    xp_next = F @ x_upd
    P_next = F @ P_upd @ F.T + Q
    return xp_next, P_next 
  \end{lstlisting}
  \end{pythonbox}
\end{minipage}%
\hfill
\begin{minipage}[t]{0.49\textwidth}
  \begin{pythonbox}[height=8.cm]
  \begin{lstlisting}[style=pythonstyle]
def approximate(F, h_fun, Q, R, x, P, z):
    H = h_fun(x, z)
    x_pred = F @ x
    P_pred = F @ P @ F.T + Q * 0.95
    y = z - H @ x_pred
    S = H @ P_pred @ H.T + R * 0.8 + 1e-12 * np.eye(len(z))
    S = np.maximum(S, 1e-14)
    inv_S = np.linalg.inv(S)
    K = P_pred @ H.T @ inv_S
    K = K * 0.65 + 0.35 * np.sign(K) * np.max(abs(K)) * np.eye(F.shape[0], len(z))
    x = x_pred + K @ y
    P = (np.eye(F.shape[0]) - K @ H) @ P_pred
    P = np.clip(P * 0.85, 1e-20, 250)
    return x, P
    
\end{lstlisting}
  \end{pythonbox}
\end{minipage}
\caption{{ Algorithms discovered for free benchmark for NSP on the left and SE on the right.}}
\label{fig:free_discovered_algorithms}
\end{figure}

Figure~\ref{fig:free_discovered_algorithms} presents representative algorithms discovered by our framework for the Free benchmark, for both next state prediction (NSP, left) and state estimation (SE, right). While both variants preserve the core Kalman filtering structure—maintaining the predict–update decomposition and the use of the innovation $y = z - Hx$ and Kalman gain $K$—they introduce several consistent non-affine modifications.  First, both algorithms apply nonlinear transformations to the innovation, including clipping and higher-order statistics (e.g., elementwise bounds or variance-dependent scaling), which act as implicit robustness mechanisms against outliers and heavy-tailed noise. Second, the covariance update incorporates adaptive damping and scaling factors (e.g., multiplicative shrinkage, clipping, or additive jitter), which stabilize matrix inversion and prevent degeneracy in ill-conditioned regimes. Third, the gain matrix is modified through data-dependent rescaling, including elementwise normalization and magnitude-based attenuation, effectively reducing the influence of unreliable measurements. Notably, the NSP variant (left) emphasizes stability in the propagated covariance through conservative updates and residual smoothing, while the SE variant (right) introduces more aggressive, input-dependent corrections to the Kalman gain and covariance, reflecting the different objectives of prediction versus filtering. Despite their simplicity, these learned modifications consistently improve performance across benchmarks, suggesting that small, structured deviations from affine updates are sufficient to capture non-Gaussian effects and model mismatch.

\section{Mot 20}

\label{Appendix:mot_20_prerformance}
MOT20 is a benchmark consisting of 8 sequences introduced in addition to prior MOT 17 \cite{leal2015motchallenge}, specifically designed to capture extremely crowded and challenging scenes. The dataset enables the evaluation of multi-object tracking methods under severe crowd density and occlusion conditions.
\begin{table}[h]
\centering
\setlength{\tabcolsep}{4.3pt}
\begin{tabular}{|c|c|c|c|c|c|c|c|}
\hline
 & KF
 & OKF
 & KE
 & OKE
 & LSTM
 & KN
 & OBS \\
\hline

NSP 
& 0.5578
& 0.4986
& \textbf{0.4599}
& 0.4622
& 0.6799
& 1.22
& 1.881 \\
\hline
\end{tabular}
\caption{RMSE comparison of different methods for MOT dataset for Next State Prediction.}

\label{tab:mot_performance}
\end{table}

Table \ref{tab:mot_performance} shows that both KalmanNet (KN) and LSTM perform significantly worse than the classical Kalman filtering approaches. In particular, KN (1.22) and LSTM (0.6799) yield substantially higher error values compared to the standard Kalman Filter (KF, 0.5578) and its optimized variants (OKF, KE, OKE), with KE achieving the best performance (0.4599). This indicates that, in this setting, neural network based approaches such as KalmanNet and LSTM fail to match the accuracy of model-based Kalman Filters, likely due to the difficulty of generalizing under the given conditions.

\begin{table}[h]
\centering
\setlength{\tabcolsep}{4.3pt}
\begin{tabular}{|c|c|c|c|c|c|c|c|}
\hline
 & KF
 & OKF
 & KE
 & OKE
 & LSTM
 & KN
 & OBS \\
\hline

NSP  
& 88.593
& 88.593
& 218.518
& 122.121
& 488.1
& \textbf{41.2}
& 63.99 \\
\hline
\end{tabular}

\caption{Test time in seconds of different methods}
\label{tab:mot_time}
\end{table}

Table \ref{tab:mot_time} indicate that Kalman Filter–based methods consistently achieve better performance while remaining computationally efficient. In particular, KF and its variants outperform learning-based approaches such as LSTM and KalmanNet in terms of accuracy, while also maintaining lower or comparable runtimes. Moreover, the KE and OKE variants further improve estimation performance over the standard KF without significantly increasing computational cost, demonstrating an effective balance between accuracy and efficiency.

\section{Convergence}

\begin{figure}[h]
    \centering

    \begin{subfigure}[t]{0.48\textwidth}
        \centering
        \includegraphics[width=\linewidth]{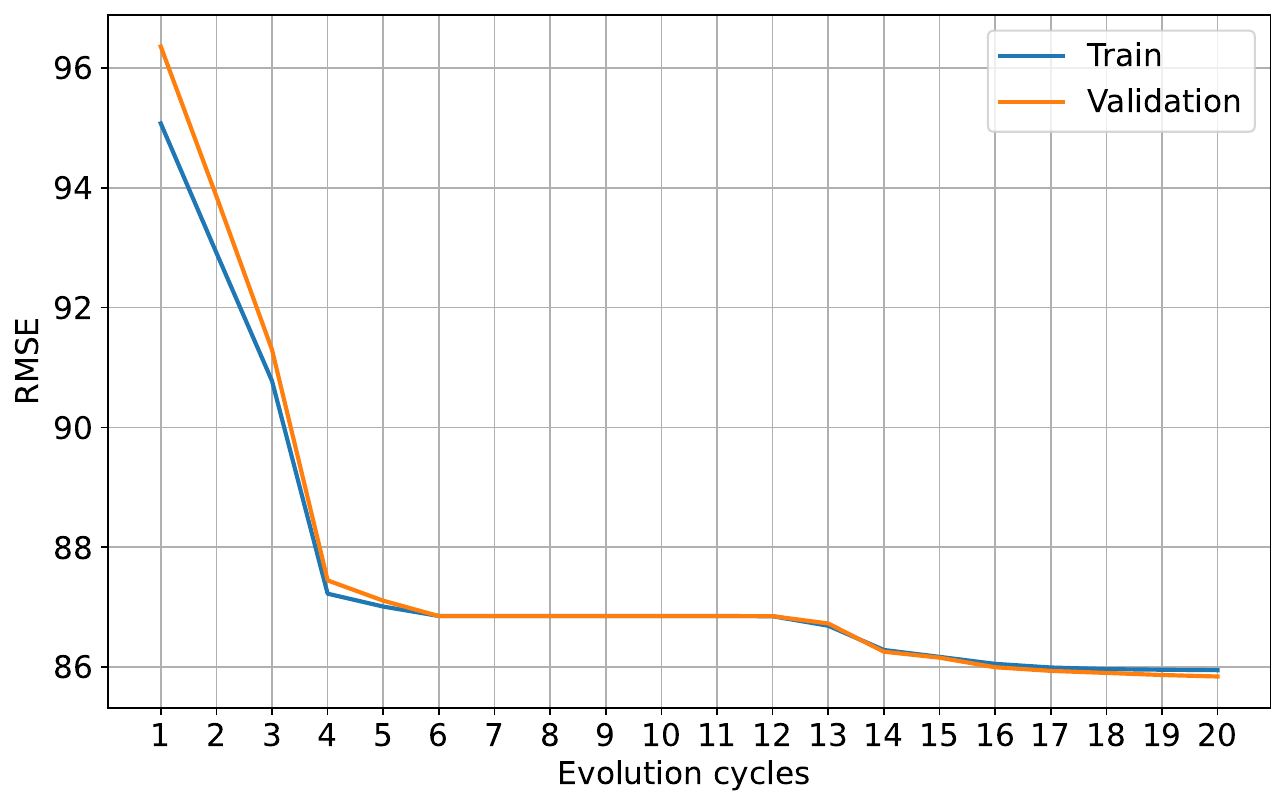}
        \caption{Free}
        \label{fig:ke}
    \end{subfigure}
    \hfill
    \begin{subfigure}[t]{0.48\textwidth}
        \centering
        \includegraphics[width=\linewidth]{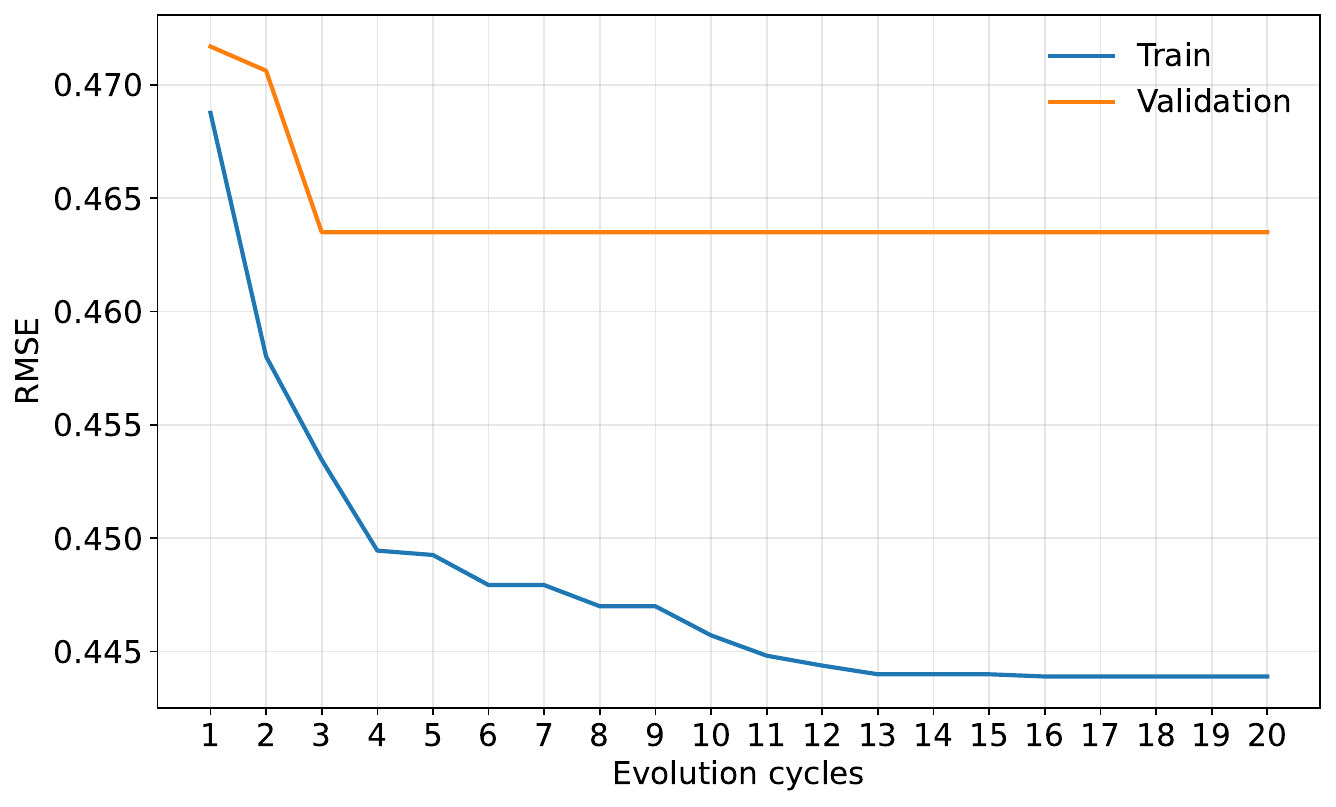}
        \caption{Mot}
        \label{fig:oke}
    \end{subfigure}

    \caption{Generalization study in next state prediction}
    \label{fig:curves}
\end{figure}
Figure \ref{fig:curves} demonstrates convergence by reporting the best-so-far performance across evolutionary cycles. Most gains occur in early iterations, followed by a plateau. In the \emph{Free} setting, training and validation curves remain aligned, indicating stable generalization. Moreover, on the free dataset, the validation loss continues to decrease, indicating that further training iterations may positively impact the final model performance.”  In the \emph{Mot} setting, validation performance saturates early, while later cycles produce solutions yielding \texttt{NaN} values at test time, suggesting overfitting to the training set through numerically unstable updates.

\end{document}